\documentclass{article}
\usepackage{microtype}
\usepackage{booktabs}
\usepackage{hyperref}

\usepackage[accepted]{icml2019}
\usepackage[utf8]{inputenc}
\usepackage{url}  
\usepackage{graphicx}  
\usepackage{amsfonts}
\usepackage{amsmath,amssymb}
\usepackage{subcaption}
\usepackage{multirow}
\usepackage{amsthm,amssymb}

\icmlsetsymbol{equal}{*}

\newtheorem{proposition}{Proposition}
\newtheorem{corollary}{Corollary}
\newtheorem{lemma}{Lemma}

\DeclareMathOperator*{\E}{\mathbb{E}}
\DeclareMathOperator*{\argmax}{arg\,max}

\newcommand{\var}{\mathrm{Var}}

\icmltitlerunning{Off-Policy Actor-Critic for Maximum General Entropy and Effective Environment Exploration}

\begin{document}


\twocolumn[
\icmltitle{Off-Policy Actor-Critic in an Ensemble: Achieving Maximum General Entropy and Effective Environment Exploration in Deep Reinforcement Learning}

\icmlsetsymbol{equal}{*}

\begin{icmlauthorlist}
\icmlauthor{Gang Chen}{to}
\icmlauthor{Yiming Peng}{to}
\end{icmlauthorlist}

\icmlaffiliation{to}{School of Engineering and Computer Science, Victoria University of Wellington, New Zealand}

\icmlcorrespondingauthor{Gang Chen}{aaron.chen@ecs.vuw.ac.nz}

\vskip 0.3in
]

\printAffiliationsAndNotice{}

\begin{abstract}
We propose a new policy iteration theory as an important extension of soft policy iteration and Soft Actor-Critic (SAC), one of the most efficient model free algorithms for deep reinforcement learning. Supported by the new theory, arbitrary entropy measures that generalize Shannon entropy, such as Tsallis entropy and R\'enyi entropy, can be utilized to properly randomize action selection while fulfilling the goal of maximizing expected long-term rewards. Our theory gives birth to two new algorithms, i.e., Tsallis entropy Actor-Critic (TAC) and R\'enyi entropy Actor-Critic (RAC). Theoretical analysis shows that these algorithms can be more effective than SAC. Moreover, they pave the way for us to develop a new Ensemble Actor-Critic (EAC) algorithm in this paper that features the use of a bootstrap mechanism for deep environment exploration as well as a new value-function based mechanism for high-level action selection. Empirically we show that TAC, RAC and EAC can achieve state-of-the-art performance on a range of benchmark control tasks, outperforming SAC and several cutting-edge learning algorithms in terms of both sample efficiency and effectiveness.
\end{abstract}

\section{Introduction}
\label{sec-intro}

The effectiveness of model-free reinforcement learning (RL) algorithms has been demonstrated extensively on robotic control tasks, computer video games, and other challenging problems \cite{lillicrap2015,mnih2015}. Despite of widespread success, many existing RL algorithms must process a huge number of environment samples in order to learn effectively \cite{mnih2016icml,wu2017nips,bhatnagar2009automatica}. Aimed at significantly reducing the \emph{sample cost}, \emph{off-policy} algorithms that can learn efficiently by reusing past experiences have attracted increasing attention \cite{lillicrap2015,munos2016nips,gu2016icra,wang2016arxiv}. Unfortunately, on RL problems with high-dimensional continuous state spaces and action spaces, off-policy learning can be highly unstable and often diverge even with small changes to hyper-parameter settings \cite{lillicrap2015,henderson2017arxiv}.

In view of the difficulties faced by existing off-policy algorithms, innovative techniques have been developed lately by seamlessly integrating both the maximum reward and the maximum entropy objectives, resulting in a new family of \emph{maximum entropy RL} algorithms \cite{nachum2017,donoghue2017,nachum2017arxiv}. A cutting-edge member of this family is the \emph{Soft Actor-Critic} (SAC) algorithm \cite{haarnoja2018icml}. SAC allows RL agents to effectively reuse past experiences by adopting an off-policy learning framework derived from the policy iteration method \cite{sutton1998book}. It also stabilizes RL by learning maximum entropy policies that are robust to environmental uncertainties and erroneous parameter settings \cite{ziebart2010thesis}. This algorithm is simpler to implement and more reliable to use than some of its predecessors \cite{haarnoja2017}.

While SAC is well-known for its high sample efficiency, the policies trained by SAC to maximize Shannon entropy never prevent an agent from exploring actions with low expected long-term rewards. This will inevitably reduce the effectiveness of environment exploration and affect the learning performance. To address this limitation, during the RL process, it can be more favorable in practice to maximize general entropy measures such as Tsallis entropy \cite{tsallis1994} or R\'enyi entropy \cite{jizba2004phy}. Specifically, it is shown mathematically in \cite{lee2018iral,chen20182} that, by maximizing Tsallis entropy, an RL agent will completely ignore unpromising actions, thereby achieving highly efficient environment exploration. However, as far as we know, no RL algorithms have ever been developed to maximize these general entropy measures in continuous action spaces.

In this paper, we present a new policy iteration theory as an important extension of soft policy iteration proposed in \cite{haarnoja2018icml} to enable Actor-Critic RL that is capable of maximizing arbitrary general entropy measures. Guided by our new theory, we have further developed two new RL algorithms to fulfill respectively the objectives to maximize Tsallis entropy and R\'enyi entropy. The potential performance advantage of our new algorithms, in comparison to SAC, will also be analyzed theoretically.

In particular, our new algorithms promote varied trade-offs between exploration and exploitation. When they are used to train an ensemble of policies which will be utilized jointly to guide an agent's environment exploration, the chance for the agent to learn high-quality policies is expected to be noticeably enhanced. Driven by this idea, a new ensemble Actor-Critic algorithm is further developed in this paper. In this algorithm, each policy in the ensemble can be trained to maximize either Tsallis entropy or R\'enyi entropy. We adopt a bootstrap mechanism proposed in \cite{osband20161} with the aim to realize deep environment exploration. To achieve satisfactory testing performance, we also introduce a new action-selection Q-network to perform high-level action selection based on actions recommended by each policy in the ensemble. On six difficult benchmark control tasks, our algorithms have been shown to clearly outperform SAC and several state-of-the-art RL algorithms in terms of both sample efficiency and effectiveness.

\section{Related Work}
\label{sec-survey}

This paper studies the collective use of three RL frameworks, i.e. the \emph{actor-critic framework} \cite{deisenroth2013}, the \emph{maximum entropy framework} \cite{nachum2017,donoghue2017,haarnoja2017,haarnoja2018icml} and the \emph{ensemble learning framework} \cite{osband20161,wiering2008smc}. Existing works such as SAC have considered the first two frameworks. It is the first time in the literature for us to further incorporate the third framework for effective environment exploration and reliable RL.

The actor-critic framework is commonly utilized by many existing RL algorithms, such as TRPO \cite{schulman2015icml}, PPO \cite{schulman20171,chen20181} and ACKTR \cite{wu2017nips}. However, a majority of these algorithms must collect a sequence of new environment samples for each learning iteration. Our algorithms, in comparison, can reuse past samples efficiently for policy training. Meanwhile, entropy regularization in the form of KL divergence is often exploited to restrict behavioral changes of updated policies so as to stabilize learning \cite{schulman2015icml,chen2018ijcnn}. On the other hand, our algorithms encourage entropy maximization for effective environment exploration.

Previous works have also studied the maximum entropy framework for both on-policy and off-policy learning \cite{nachum2017,donoghue2017,nachum2017arxiv}. Among them, a large group focused mainly on problems with discrete actions. Some recently developed algorithms have further extended the maximum entropy framework to continuous domains \cite{haarnoja2017,haarnoja2018icml}. Different from these algorithms that focus mainly on Shannon entropy, guided by a new policy iteration theory, our newly proposed RL algorithms can maximize general entropy measures such as Tsallis entropy and R\'enyi entropy in continuous action spaces.

In recent years, ensemble methods have gained wide-spread attention in the research community \cite{osband20161,chen2017ucb,bukcman2018}. For example, Bootstrapped Deep Q-Network and UCB Q-Ensemble \cite{osband20161,chen2017ucb} have been developed for single-agent RL with success. However, these algorithms were designed to work with discrete actions. In an attempt to tackle continuous problems, some promising methods such as ACE \cite{huang2017corr}, SOUP \cite{zeng2018ijcai} and MACE \cite{peng2016siggraph} have been further developed. Several relevant algorithms have also been introduced to support ensemble learning under a multi-agent framework \cite{lowe2017nips,panait2005aamas}.

Different from ACE that trains each policy in the ensemble independently, our ensemble algorithm trains all policies by using the same replay buffer so as to significantly reduce sample cost. Different from MACE that treats policies in the ensemble as individual actions to be explored frequently across successive interactions with the learning environment, we adopt the bootstrap mechanism to realize deep environment exploration. Moreover, to the best of our knowledge, different from SOUP and other relevant algorithms \cite{kalweit2017pmlr}, the idea of training a separate Q-network for high-level action selection has never been studied before for single-agent RL. The effectiveness of this method is also theoretically analyzed in Section \ref{sec-algorithm}.

\section{Preliminaries}
\label{sec-back}

The basic concepts and notations for RL will be introduced in this section. A general maximum entropy learning framework will also be presented.

\subsection{The Reinforcement Learning Problem}
\label{sub-rl}

We focus on RL with continuous state spaces and continuous action spaces. At any time $t$, an agent can observe the current state of its learning environment, denoted as $s_t\in\mathbb{S}\subseteq \mathbb{R}^n$, in a $n$-dimensional state space. Based on the state observation, the agent can perform an action $a_t$ selected from an $m$-dimensional action space $\mathbb{A}\subseteq\mathbb{R}^m$. This causes a state transition in the environment from state $s_t$ to a new state $s_{t+1}$, governed by the probability distribution $P(s_t,s_{t+1},a_t)$ which is unknown to the agent. Meanwhile, an immediate feedback in the form of a scalar and bounded reward, i.e. $r(s_t,a_t)$, will be provided for the agent to examine the suitability of its decision to perform $a_t$. Guided by a policy $\pi(s,a)$ that specifies the probability distribution of performing any action $a$ in any state $s$, the agent has the goal in RL to maximize its long-term cumulative rewards, as described below.
\begin{equation}
\max_{\pi} \E_{(s_t,a_t)\sim \pi} \sum_{t=0}^{\infty} \gamma^t r(s_t,a_t)
\label{eq-lt-cum-rew}
\end{equation}
\noindent
where $\gamma\in[0,1)$ is a discount factor for the RHS of \eqref{eq-lt-cum-rew} to be well-defined. Meanwhile, the expectation is conditional on policy $\pi$ that guides the agent to select its actions in every state $s_t$ it encounters. Through RL, the agent is expected to identify an optimal policy, denoted as $\pi^*$, that solves the maximization problem in \eqref{eq-lt-cum-rew} above.

\subsection{A Maximum Entropy Learning Framework}
\label{sub-me-lf}

While maximizing the long-term cumulative rewards in \eqref{eq-lt-cum-rew}, by considering simultaneously a maximum entropy objective, an RL agent can enjoy several key advantages in practice, such as effective environment exploration as well as improved learning speed \cite{ziebart2010thesis,haarnoja2018icml}. Driven by this idea, \eqref{eq-lt-cum-rew} can be extended to become
\begin{equation}
\max_{\pi} \E_{(s_t,a_t)\sim \pi} \sum_{t_0}^{\infty} \gamma^t r(s_t,a_t) + \alpha \gamma^t \mathcal{H}\left( \pi(s_t,\cdot) \right)
\label{eq-lt-cum-rew-ext}
\end{equation}
\noindent
where $\alpha$ as a scalar factor controls the relative importance of the entropy term against the cumulative reward in \eqref{eq-lt-cum-rew-ext}. In this way, we can further manage the stochasticity of the optimal policy. In the past few years, researchers have studied extensively the use of Shannon entropy given below for maximum entropy RL.
\begin{equation}
\mathcal{H}^s\left( \pi(s,\cdot) \right)=\int_{a\in\mathbb{A}} -\pi(s,a)\log\pi(s,a)\mathrm{d}a
\label{eq-shannon-entropy}
\end{equation}
Some recent works have also explored the maximization of Tsallis entropy \cite{lee2018iral,chow2018icml}. As a generalization of Shannon entropy, Tsallis entropy is derived from the $q$-logarithm function
\begin{equation}
\log_{(q)}\pi(s,\cdot)=\frac{\pi(s,\cdot)^{q-1}-1}{q-1}
\label{eq-gen-log}
\end{equation}
\noindent
where $q\geq 1$ is the \emph{entropic index}. It is not difficult to verify that $\lim_{q\rightarrow 1}\log_{(q)}\pi(s,\cdot)=\log\pi(s,\cdot)$. Subsequently Tsallis entropy can be defined below
\begin{equation}
\mathcal{H}^q\left( \pi(s,\cdot) \right)= \int_{a\in\mathbb{A}}-\pi(s,a)\frac{\pi(s,a)^{q-1}-1}{q-1}\mathrm{d}a
\label{eq-tsallis-entropy}
\end{equation}
Existing research demonstrated the effectiveness of Tsallis entropy on RL problems with discrete actions \cite{lee2018iral}. Particularly, Tsallis entropy enables an agent to completely ignore unpromising actions for highly efficient environment exploration. In this paper, we will further consider the application of Tsallis entropy in continuous action spaces, a question that has never been studied in the past. Moreover, we will develop a general theory for maximum entropy RL that supports arbitrary entropy measures. To testify the wide applicability of our theory, we will develop a new RL algorithm that maximizes the famous R\'enyi entropy, as defined below.
\begin{equation}
\mathcal{H}^{\eta}\left( \pi(s,\cdot) \right)=\frac{1}{1-\eta}\log\left( \int_{a\in\mathbb{A}} \pi(s,a)^{\eta} \mathrm{d}a \right)
\label{eq-renyi-entropy}
\end{equation}
\noindent
where $\eta\geq 1$ is the entropic index. Although R\'enyi entropy plays an important role in many fields such as quantum physics and theoretical computer science, we are not aware of any existing RL algorithms that utilize R\'enyi entropy to control stochastic environment exploration.

\section{Policy Iteration for Maximum Entropy Reinforcement Learning}
\label{sec-theory}

Aimed at developing a new theory for maximum entropy RL, we adopt a policy iteration framework which can be subsequently transformed into practical RL algorithms. Under this framework, RL is realized through iterative execution of \emph{policy evaluation} and \emph{policy improvement} steps. The policy evaluation step is responsible for learning the value functions of a given policy. Based on the learned value functions, improvement over existing policies can be performed subsequently during the policy improvement steps. In line with this framework, for any fixed policy $\pi$, the Q-function and V-function for $\pi$ can be defined respectively as
\begin{equation}
\begin{split}
&Q^{\pi}(s,a)=\\&\E_{\begin{array}{c}(s_t,a_t)\sim\pi,\\s_0=a,a_0=a\end{array}}\left[ \sum_{t=0}^{\infty}\gamma^t r(s_t,a_t) + \gamma^{t+1} \alpha \mathcal{H}\left( \pi(s_{t+1},\cdot) \right) \right]
\end{split}
\label{eq-q-func}
\end{equation}
\noindent
and
\begin{equation}
V^{\pi}(s)=\E_{a\sim \pi(s,\cdot)}\left[ Q(s,a)+\alpha \mathcal{H}\left( \pi(s,\cdot) \right)\right]
\label{eq-v-func}
\end{equation}
\noindent
In order to learn both $Q^{\pi}$ and $V^{\pi}$, the Bellman backup operator $\mathcal{T}^{\pi}$ as presented below is very useful
\begin{equation}
\mathcal{T}^{\pi}Q^{\pi}(s,a)=r(s,a)+\gamma \E_{s'\sim P(s,s',a)} \left[ V^{\pi}(s') \right]
\label{eq-bellman}
\end{equation}
In fact, it can be shown that the process of learning $Q^{\pi}$ (and subsequently $V^{\pi}$) through $\mathcal{T}^{\pi}$ can converge to the true value functions of policy $\pi$, as presented in Proposition \ref{proposition-1} below (see Appendix A for proof).
\begin{proposition}
Starting from arbitrary Q-function $Q^0$ and define $Q^{k+1}=\mathcal{T}^{\pi}Q^k$, repeated updating of the Q-function through the Bellman backup operator $\mathcal{T}^{\pi}$ in \eqref{eq-bellman} will result in $Q^k$ converging to $Q^{\pi}$ for any policy $\pi$ as $k\rightarrow\infty$.
\label{proposition-1}
\end{proposition}

In the policy improvement step, a mechanism must be developed to build a new policy $\pi'$ based on an existing policy $\pi$ such that $Q^{\pi'}(s,a)\geq Q^{\pi}(s,a)$ for any $s\in\mathbb{S}$ and $a\in\mathbb{A}$. In SAC, $\pi'(s,a)$ is constructed to approximate the distribution derived from the exponential of $Q^{\pi}(s,a)$. This is achieved by minimizing the KL divergence below.
\begin{equation}
\min_{\pi^{\prime}\in\Pi} D_{KL}\left( \pi^{\prime}(s,\cdot),\exp\left( Q^{\pi}(s,\cdot)-\mathbb{C}_s \right)  \right),\forall s\in\mathbb{S}
\label{eq-pi-improve-sac}
\end{equation}
\noindent
where $\Pi$ refers to a predefined family of policies and may vary significantly from one problem domain to another. $\mathbb{C}_s$ normalizes the second argument of the KL divergence in order for it to be a well-defined probability distribution. It is important to note that the policy improvement mechanism given in \eqref{eq-pi-improve-sac} only works when $\mathcal{H}(\pi)=\mathcal{H}^s(\pi)$ in \eqref{eq-lt-cum-rew-ext}. Driven by our goal to maximize arbitrary entropy measures during RL, we propose a new policy improvement mechanism in the form of a maximization problem as follows.
\begin{equation}
\max_{\pi^{\prime}\in\Pi} \E_{a\sim \pi^{\prime}(s,\cdot)}\left[ Q^{\pi}(s,a)+\alpha \mathcal{H}\left( \pi^{\prime}(s,\cdot) \right)\right],\forall s\in\mathbb{S}
\label{eq-pi-improve-gac}
\end{equation}
\noindent
Based on \eqref{eq-pi-improve-gac}, as formalized in Proposition \ref{proposition-2} (see Appendix B for proof), it can be confirmed that in a tabular setting the policy improvement step can always produce a new policy $\pi'$ that is either better than or equivalent in performance to an existing policy $\pi$.
\begin{proposition}
Let $\pi$ be an existing policy and $\pi^{\prime}$ be a new policy that optimizes \eqref{eq-pi-improve-gac}, then $Q^{\pi^{\prime}}(s,a)\geq Q^{\pi}(s,a)$ for all $s\in\mathbb{S}$ and $a\in\mathbb{A}$.
\label{proposition-2}
\end{proposition}

Guided by Propositions \ref{proposition-1} and \ref{proposition-2}, a full policy iteration algorithm can be established to alternate between the policy evaluation and policy improvement steps. As this process continues, we have the theoretical guarantee that, in the tabular case, RL will converge to the optimal stochastic policy among the family of policies $\Pi$, as presented in Proposition \ref{proposition-3} below (see Appendix C for proof).
\begin{proposition}
The policy iteration algorithm driven by the Bellman operator in \eqref{eq-bellman} and the policy improvement mechanism in \eqref{eq-pi-improve-gac} will converge to a policy $\pi^*\in\Pi$ such that $Q^{\pi^*}(s,a)\geq Q^{\pi}(s,a)$ among all $\pi\in\Pi$, and for all $s\in\mathbb{S}$ and $a\in\mathbb{A}$.
\label{proposition-3}
\end{proposition}

The key difference of our policy iteration theory from the theory developed in \cite{haarnoja2018icml} lies in the new policy improvement mechanism in \eqref{eq-pi-improve-gac}. Particularly, \eqref{eq-pi-improve-gac} makes it straightforward to improve an existing policy $\pi$ by maximizing both its value function $Q^{\pi}$ and an arbitrary entropy measure. Meanwhile, Corollary \ref{corollary-1} (see Appendix D for proof) below shows that \eqref{eq-pi-improve-gac} and \eqref{eq-pi-improve-sac} are equivalent when the objective is to maximize Shannon entropy $\mathcal{H}^s$. Hence the theoretical results presented in this section, especially the newly proposed policy improvement mechanism, stand for an important generalization of SAC.
\begin{corollary}
Consider any policy $\pi\in\Pi$. When $\mathcal{H}(\pi)=\mathcal{H}^s(\pi)$ in \eqref{eq-lt-cum-rew-ext}, let the policy obtained from solving \eqref{eq-pi-improve-sac} be $\pi^{\prime}\in\Pi$ and the policy obtained from solving \eqref{eq-pi-improve-gac} be $\pi^{\prime\prime}\in\Pi$. Then $Q^{\pi^{\prime}}(s,a)=Q^{\pi^{\prime\prime}}(s,a)$ all $s\in\mathbb{S}$ and $a\in\mathbb{A}$.
\label{corollary-1}
\end{corollary}

\section{Actor-Critic Algorithms for Maximum Entropy Reinforcement Learning}
\label{sec-algorithm}

In this section, we will first develop two new RL algorithms to maximize respectively Tsallis entropy and R\'enyi entropy. Afterwards, we will further propose an ensemble algorithm that simultaneously trains multiple policies through actor-critic learning techniques and analyze its performance advantage.

\subsection{Building Actor-Critic Algorithms to Maximize Tsallis and R\'eny Entropy}

Driven by the policy iteration theory established in Section \ref{sec-theory} and following the design of SAC in \cite{haarnoja2018icml}, our new RL algorithms will contain two main components, i.e. the \emph{actor} and the \emph{critic}. The critic manages value function learning and the actor is in charge of policy learning. On large-scale RL problems with continuous state spaces and continuous action spaces, it is unfeasible to perform policy evaluation and policy improvement till convergence for every round of policy update. Instead, both the actor and the critic must learn concurrently, as shown in Algorithm \ref{algorithm-1}.
\begin{algorithm}[!ht]
 \begin{algorithmic}
 \STATE {\bf Input}: three DNNs, i.e. $V_{\theta}(s)$, $Q_{\omega}(s,a)$ and $\pi_{\phi}(s,a)$, and a replay buffer $\mathcal{B}$ that stores past state-transition samples for training.
 \STATE {\bf for} each environment step $t$ {\bf do}:
 \STATE \ \ \ \ Sample and perform $a_t\sim \pi_{\phi}(s_t,\cdot)$
 \STATE \ \ \ \ Add $(s_t,a_t,s_{t+1},r_t)$ to $\mathcal{B}$
  \STATE \ \ \ \ {\bf for} each learning step {\bf do}:
 \STATE \ \ \ \ \ \ \ \ Sample a random batch $\mathcal{D}$ from $\mathcal{B}$.
 \STATE \ \ \ \ \ \ \ \ Perform critic learning:
 \STATE \ \ \ \ \ \ \ \ \ \ \ \ $\theta\leftarrow \theta-\lambda_{\theta}\nabla_{\theta}\epsilon_{\theta}^{\mathcal{D}}$
 \STATE \ \ \ \ \ \ \ \ \ \ \ \ $\omega\leftarrow\omega-\lambda_{\omega}\nabla_{\omega}\epsilon_{\omega}^{\mathcal{D}}$
 \STATE \ \ \ \ \ \ \ \ Perform actor learning:
 \STATE \ \ \ \ \ \ \ \ \ \ \ \ $\phi\leftarrow\phi+\lambda_{\phi}\nabla_{\phi}\epsilon_{\phi}^{\mathcal{D}}$
 \end{algorithmic}
\caption{The actor-critic algorithm for maximum entropy RL.}
\label{algorithm-1}
\end{algorithm}

In line with this algorithm design, we will specifically consider parameterized value functions and policies, i.e. $V_{\theta}(s)$, $Q_{\omega}(s,a)$ and $\pi_{\phi}(s,a)$, in the form of \emph{Deep Neural Networks} (DNNs), where the network parameters are $\theta$, $\omega$ and $\phi$ respectively. As explained in \cite{haarnoja2018icml}, although it is not necessary for the critic to learn both $V_{\theta}(s)$ and $Q_{\omega}(s,a)$, the explicitly learned $V_{\theta}(s)$ can noticeably stabilize policy training by serving as the state-dependent baseline function. Hence $V_{\theta}(s)$ is also included in the design of our new actor-critic algorithms. Depending on the actual entropy measure used in \eqref{eq-lt-cum-rew-ext}, the learning rules to be employed for updating network parameters can be very different. We will derive these rules in the sequel.

Similar to SAC, a replay buffer $\mathcal{B}$ is maintained consistently during RL. As shown in Algorithm \ref{algorithm-1}, at each learning step, a fixed-size batch $\mathcal{D}\subseteq\mathcal{B}$ of state-transition samples can be collected from $\mathcal{B}$ and used to define the squared error below for the purpose of training $V_{\theta}$.
\begin{equation}
\begin{split}
\epsilon_{\theta}^{\mathcal{D}}=&\frac{1}{2 \|\mathcal{D}\|}\cdot  \\
& \sum_{(s,a,s',r)\in\mathcal{D}} \left( V_{\theta}(s)-\E_{b\sim\pi(s,\cdot)}Q_{\omega}(s,b)-\alpha\mathcal{H}(\pi)\right)^2
\label{eq-epsilon-theta}
\end{split}
\end{equation}
\noindent
where $\|\mathcal{D}\|$ is the cardinality of $\mathcal{D}$. Similarly, the Bellman residue for training $Q_{\omega}$ can be determined as
\begin{equation}
\epsilon_{\omega}^{\mathcal{D}}=\frac{1}{2\|\mathcal{D}\|}\sum_{(s,a,s',r)\in\mathcal{D}}\left( Q_{\omega}(s,a)-r-\gamma V_{\theta}(s') \right)^2
\label{eq-epsilon-omega}
\end{equation}
Based on \eqref{eq-epsilon-theta} and \eqref{eq-epsilon-omega}, $\nabla_{\theta}\epsilon_{\theta}^{\mathcal{D}}$ and $\nabla_{\omega}\epsilon_{\omega}^{\mathcal{D}}$ can be further utilized to build the learning rules for $\theta$ and $\omega$ respectively (See Algorithm \ref{algorithm-1}). Specifically,
\begin{equation}
\begin{split}
&\nabla_{\theta}\epsilon_{\theta}^{\mathcal{D}}=\frac{1}{\|\mathcal{D}\|}\cdot \\
& \sum_{(s,a,s',r)\in\mathcal{D}}\nabla_{\theta}V_{\theta}(s)\left( \begin{array}{l} {\displaystyle V_{\theta}(s)- \E_{b\sim\pi(s,\cdot)}Q_{\omega}(s,b)} \\ -\alpha\mathcal{H}(\pi_{\phi})\end{array}\right)
\end{split}
\label{eq-nabla-epsilon-theta}
\end{equation}
\noindent
To evaluate $\nabla_{\theta}\epsilon_{\theta}^{\mathcal{D}}$ in \eqref{eq-nabla-epsilon-theta}, we must estimate the expectation $\E_{b\sim\pi(s,\cdot)}Q_{\omega}(s,b)$ and the entropy term $\mathcal{H}(\pi_{\phi})$ efficiently. In practice, both of the two in \eqref{eq-nabla-epsilon-theta} can be quickly approximated through a group of $k$ actions, i.e. $a_1,\ldots,a_k$, which will be sampled independently from $\pi_{\phi}(s,\cdot)$. In particular, when Tsallis entropy $\mathcal{H}^q(\pi)$ in \eqref{eq-tsallis-entropy} is exploited for maximum entropy RL, then
\begin{equation*}
\mathcal{H}^q\left(\pi_{\phi}(s,\cdot)\right)\approx-\frac{1}{k}\sum_{i=1}^k \frac{\pi_{\phi}(s,a_i)^{q-1}-1}{q-1}
\label{eq-tsallis-est}
\end{equation*}
\noindent
On the other hand, when R\'enyi entropy $\mathcal{H}^{\eta}(\pi_{\phi})$ in \eqref{eq-renyi-entropy} is adopted, we must either calculate $\int\pi^{\eta}(s,a)\mathrm{d}a$ precisely or estimate it through multiple action samples, depending on the type of action-selection distributions supported by the corresponding policy family $\Pi$. Please refer to Appendix E for detailed discussion on the techniques we used to evaluate $\int\pi^{\eta}(s,a)\mathrm{d}a$ and $\mathcal{H}^{\eta}(\pi_{\phi})$.

Besides the critic, we must develop the learning rules for the actor too. According to the policy improvement mechanism presented in \eqref{eq-pi-improve-gac}, the performance index below can be employed for training $\pi_{\phi}$.
\begin{equation}
\epsilon_{\phi}^{\mathcal{D}}=\frac{1}{\|D\|}\sum_{(s,a,s',r)\in\mathcal{D}} \E_{a\sim\pi_{\phi}(s,\cdot)} Q_{\omega}(s,a)+\alpha\mathcal{H}(\pi_{\phi})
\label{eq-epsilon-policy}
\end{equation}
\noindent
Subject to the entropy measure used in \eqref{eq-epsilon-policy}, $\nabla_{\phi}\epsilon_{\phi}^{\mathcal{D}}$ must be computed differently. Specifically, with the objective of maximum Tsallis entropy $\mathcal{H}^q(\pi_{\phi})$, we will first determine $\nabla_{\phi}\epsilon_{\phi}$ in every sampled state $s$, i.e. $\nabla_{\phi}\epsilon_{\phi}^s$, as given below.
\begin{equation}
\begin{split}
&\nabla_{\phi}\epsilon_{\phi}^s=\int_{a\in\mathcal{A}}\pi_{\phi}(s,a) \nabla_{\phi}\log\pi_{\phi}(s,a) Q_{\omega}(s,a)\mathrm{d}a- \\
&\alpha\int_{a\in\mathcal{A}}\pi_{\phi}(s,a)\frac{q\ \pi_{\phi}(s,a)^{q-1}-1}{q-1}\nabla_{\phi}\log\pi_{\phi}(s,a)\mathrm{d}a
\end{split}
\label{eq-nabla-phi-tsallis}
\end{equation}
\noindent
Because $\int_{a\in\mathcal{A}}\pi_{\phi}(s,a)\nabla_{\phi}\log\pi_{\phi}(s,a)\mathrm{d}a=0$ in any state $s\in\mathbb{S}$, the second integral term at the RHS of \eqref{eq-nabla-phi-tsallis} can be re-written as
\[
q \int_{a\in\mathcal{A}}\pi_{\phi}(s,a)\nabla_{\phi}\pi_{\phi}(s,a)\log_{(q)}\pi_{\phi}(s,a)\mathrm{d}a
\]
\noindent
As a consequence, $\nabla_{\phi}\epsilon_{\phi}^s$ will be approximated as
\begin{equation}
\begin{split}
& \nabla_{\phi}\epsilon_{\phi}^s\approx \\
& \frac{1}{k}\sum_{i=1}^k \nabla_{\phi}\log\pi_{\phi}(s,a_i)\left( A(s,a_i)-\alpha q \log_{(q)}\pi_{\phi}(s,a_i) \right)
\end{split}
\label{eq-nabla-s-tsallis}
\end{equation}
\noindent
where $a_1,\ldots,a_k$ are $k$ actions sampled from $\pi_{\phi}(s,\cdot)$ and
\[
A(s,a_i)=Q_{\omega}(s,a_i)-V_{\theta}(s)
\]
\noindent
Hence
\begin{equation}
\nabla_{\phi}\epsilon_{\phi}^{\mathcal{D}}\approx \frac{1}{\|\mathcal{D}\|}\nabla_{\phi}\epsilon_{\phi}^s
\label{eq-nabla-phi}
\end{equation}
In the same vein, $\nabla_{\phi}\epsilon_{\phi}^{\mathcal{D}}$ can also be computed easily when R\'enyi entropy $\mathcal{H}^{\eta}(\pi_{\phi})$ is adopted. Specifically, in any sampled state $s$,
\begin{equation*}
\begin{split}
\nabla_{\phi}\epsilon_{\phi}^s=&\int_{a\in\mathcal{A}}Q_{\omega}(s,a)\pi_{\phi}(s,a)\nabla_{\phi}\log\pi_{\phi}(s,a)\mathrm{d}a+\\
&\int_{a\in\mathcal{A}} \frac{\alpha \eta \pi_{\phi}(s,a)^{\eta}}{(1-\eta)\int_{a\in\mathcal{A}}\pi_{\phi}(s,a)^{\eta}\mathrm{d}a}\nabla_{\phi}\log\pi_{\phi}(s,a)\mathrm{d}a
\end{split}
\end{equation*}
\noindent
As a result, $\nabla_{\phi}\epsilon_{\phi}^{\mathcal{D}}$ will be estimated as
\begin{equation}
\begin{split}
\nabla_{\phi}\epsilon_{\phi}^s &\approx \frac{1}{k}\sum_{i=1}^k \nabla_{\phi}\log\pi_{\phi}(s,a_i)\\& \left(A(s,a_i)+\frac{\alpha\eta \pi_{\phi}(s,a_i)^{\eta-1}}{(1-\eta)\int_{a\in\mathcal{A}}\pi_{\phi}(s,a)^{\eta}\mathrm{d}a} \right)
\end{split}
\label{eq-nabla-phi-renyi}
\end{equation}
Based on \eqref{eq-nabla-s-tsallis}, \eqref{eq-nabla-phi} and \eqref{eq-nabla-phi-renyi}, two alternative learning rules have been developed for the actor to train $\pi_{\phi}$. Together with the learning rules for the critic developed previously, they give rise to two separate RL algorithms, one supports RL that maximizes Tsallis entropy and will be called \emph{Tsallis entropy Actor-Critic} (TAC) and the other enables RL for maximum R\'enyi entropy and will be named as \emph{R\'enyi entropy Actor-Critic} (RAC).

In Appendix F, theoretical analysis will be performed to study the effectiveness of TAC and RAC. Specifically, under suitable conditions and assumptions, the performance lower bounds for both TAC and RAC can be derived analytically. Moreover we show that the performance lower bound of TAC (when entropic index $q>1$) can be higher than that of SAC (when $q=1$). In other words, TAC can be more effective than SAC while enjoying efficient environment exploration through random action selection with maximum entropy. On the other hand, although the performance bound for RAC does not depend on entropic index $\eta$, nevertheless we can control the influence of the maximum entropy objective in \eqref{eq-lt-cum-rew-ext} and in \eqref{eq-pi-improve-gac} through $\eta$. Policy improvement during RL can be controlled subsequently, resulting in varied trade-offs between exploration and exploitation. It paves the way for the development of a new ensemble algorithm for single-agent RL in the next subsection.

\subsection{Building an Ensemble Actor-Critic Algorithm}
\label{sub-sec-eac}

Using TAC and RAC respectively, we can simultaneously train \emph{an ensemble of policies}. The effectiveness and reliability of RL are expected to be enhanced when an agent utilizes such an ensemble to guide its interaction with the learning environment (see Proposition \ref{proposition-ens} below). To fulfill this goal, inspired by the bootstrap mechanism for deep environment exploration \cite{osband20161}, a new \emph{Ensemble Actor-Critic} (EAC) algorithm is developed in this subsection. Algorithm \ref{algorithm-2} highlights the major steps of EAC.
\begin{algorithm}[!ht]
 \begin{algorithmic}
 \STATE {\bf Input}: an ensemble of $L$ policies and a replay buffer $\mathcal{B}$ that stores past state-transition samples for training.
 \STATE {\bf for} each problem episode {\bf do}:
 \STATE \ \ \ \ Choose a policy from the ensemble randomly
 \STATE \ \ \ \ {\bf for} $t=1,\ldots$ until end of episode {\bf do}:
 \STATE \ \ \ \ \ \ \ \ Use the chosen policy to sample action $a_t$.
 \STATE \ \ \ \ \ \ \ \ Perform $a_t$.
 \STATE \ \ \ \ \ \ \ \ Insert observed state transition into $\mathcal{B}$.
 \STATE \ \ \ \ \ \ \ \ Sample a random batch $\mathcal{D}$ from $\mathcal{B}$.
 \STATE \ \ \ \ \ \ \ \ Use TAC or RAC and $\mathcal{D}$ to train all policies in

 \ \ \ \ \ \ \ \ the ensemble.
 \STATE \ \ \ \ \ \ \ \ Use $\mathcal{D}$ to train action-selection Q-network $Q_{\psi}$.
 \end{algorithmic}
\caption{An Ensemble Actor-Critic (EAC) Algorithm for RL.}
\label{algorithm-2}
\end{algorithm}

EAC trains all policies in the ensemble by using the same replay buffer $\mathcal{B}$. Meanwhile, a policy will be chosen randomly from the ensemble to guide the agent's future interaction with the learning environment during each problem episode which starts from an initial state and ends whenever a terminal state or the maximum number of time steps has been reached. As explained in \cite{osband20161}, this bootstrap mechanism facilitates deep and efficient environment exploration.

In addition to training, a new technique must be developed to guide an RL agent to select its actions during testing. We have investigated several possible techniques for this purpose including choosing the policy with the highest training performance for action selection as well as asking every critic in the ensemble to evaluate the actions recommended by all policies and selecting the action with the highest average Q-value. The latter option is recently introduced by the SOUP algorithm \cite{zeng2018ijcai}. However these techniques do not allow EAC to achieve satisfactory testing performance. We found that this is because every policy in EAC is not only trained to maximize cumulative rewards but also to maximize an entropy measure. Therefore the evaluation of any action by a critic in the ensemble will be influenced by the entropy of the corresponding policy. While this is important for training, it is not desirable for testing. In fact during testing we must choose the most promising actions for the pure purpose of reward maximization.

Guided by this understanding, we have introduced a new component for high-level action selection in EAC, i.e. an action-selection Q-network $Q_{\psi}$ parameterized by $\psi$, as highlighted in Algorithm \ref{algorithm-2}. $Q_{\psi}$ will be trained together with all policies in the ensemble. The Bellman residue below without involving any entropy measures will be exploited for training $Q_{\psi}$.
\begin{equation*}
\epsilon_{\psi}^{\mathcal{D}}=\frac{1}{2 \|\mathcal{D}\|}\sum_{(s,a,s',r)\in\mathcal{D}} \left( \begin{array}{l}
Q_{\psi}(s,a)-r \\
-\gamma\max_{b\in\{b_1\ldots,b_L\}}Q_{\psi}(s',b)
\end{array} \right)^2
\end{equation*}
\noindent
where $b_1,\ldots,b_L$ refer to the $L$ actions sampled respectively from the $L$ policies in the ensemble for state $s'$. Consequently, during testing, every policy will recommend an action in each newly encountered state. The action with the highest Q-value according to $Q_{\psi}$ will be performed by the agent as a result. Proposition \ref{proposition-ens} (see Appendix G for proof) below gives the clue regarding why high-level action selection via $Q_{\psi}$ can be effective.
\begin{proposition}
For any policy $\pi\in\Pi$, assume that action selection guided by $\pi$ in any state $s\in\mathbb{S}$ follows multivariate normal distribution. Also assume that the Q-function for policy $\pi$ is continuous, differentiable with bounded derivatives and unimodal. Let $\pi'\in\Pi$ be a new policy that maximizes \eqref{eq-pi-improve-gac} when $\alpha=0$. Meanwhile, $\pi_1,\ldots,\pi_L$ represent the $L$ individual policies in the ensemble, each of which also maximizes \eqref{eq-pi-improve-gac} when $\alpha\neq 0$. Define the joint policy $\pi^e$ below
\[
\pi^e(s)=\argmax_{a\in\{\tilde{a}_1,\ldots,\tilde{a}_L\}} Q^{\pi}(s,a)
\]
\noindent
where $\tilde{a}_1,\ldots,\tilde{a}_L$ stand for the $L$ actions sampled respectively from each policy in the ensemble. Then, as $L\rightarrow\infty$, $Q^{\pi^e}(s,a)\geq Q^{\pi'}(s,a)$ for all $s\in\mathbb{S}$ and $a\in\mathbb{A}$.
\label{proposition-ens}
\end{proposition}

It is intuitive to consider $Q_{\psi}$ in EAC as being trained to approximate the Q-function of the joint policy $\pi^e$ in Proposition \ref{proposition-ens}. Consequently Proposition \ref{proposition-ens} suggests that, when $L$ is sufficiently large, high-level action selection guided by $\pi^e$ and therefore $Q_{\psi}$ can reach the optimal cumulative rewards achievable by any policy $\pi\in\Pi$. Meanwhile, we can continue to enjoy effective environment exploration during training through the bootstrap mechanism since the trained $Q_{\psi}$ is only exploited for action selection during testing.

\section{Experiments}
\label{sec-exp}

To examine the sample complexity and performance of TAC, RAC and EAC, we conduct experiments on six benchmark control tasks, including Ant, Half Cheetah, Hopper, Lunar Lander, Reacher and Walker2D. We rely consistently on the implementation of these benchmark problems provided by OpenAI GYM~\cite{Brockman:2016wv} and powered by the PyBullet simulator~\cite{Tan:2018us}.

Many previous works utilized the MuJoCo physics engine to simulate system dynamics of these control tasks~\cite{Todorov:2012vi}. We did not study MuJoCo problems due to two reasons. First, it is widely reported that PyBullet benchmarks are tougher to solve than MuJoCo problems~\cite{Tan:2018us}. Hence, we expect to show the performance difference among all competing algorithms more significantly on PyBullet problems. Second, PyBullet is license-free with increasing popularity. In contrast, MuJoCo is only available to its license holders. To make our experiment results reproducible, the source code of TAC, RAC and EAC has been made freely available online \footnote{https://github.com/yimingpeng/sac-master}.

There are eight competing algorithms involved in our experiments, i.e. SAC, TAC, RAC, EAC-TAC, EAC-RAC, TRPO, PPO and ACKTR. Among them EAC-TAC refers to the ensemble learning algorithm developed in Subsection \ref{sub-sec-eac} where TAC is used to train every policy in the ensemble. EAC-RAC refers to the variation where policy training is realized through RAC. Meanwhile, TRPO, PPO and ACKTR are state-of-the-art RL algorithms frequently employed for performance comparison. We used the high-quality implementation of TRPO, PPO and ACKTR provided by OpenAI Baselines\footnote{https://github.com/openai/baselines}. The source code for SAC is obtained from its inventors\footnote{https://github.com/haarnoja/sac}.

We follow closely \cite{haarnoja2018icml} to determine hyper-parameter settings of SAC, TAC, RAC, EAC-TAC and EAC-RAC. The hyper-parameter settings of TRPO, PPO and ACKTR were obtained also from the literature~\cite{schulman2015icml,schulman20171,wu2017nips}. Detailed hyper-parameter settings for all algorithms can be found in Appendix H.

We first examine the influence of entropic index on the performance of TAC and RAC. Figure \ref{fig:tac_entropic_index} depicts the learning performance of TAC with respect to three different settings of the entropic index $q$ (i.e. 1.5, 2.0 and 2.5) and also compares TAC with SAC as the baseline. As evidenced in the figure, with proper settings of the entropic index $q$, TAC can clearly outperform SAC on all six benchmark problems. Meanwhile, the best value for $q$ varies from one problem domain to another. When comparing TAC subsequently with other competing algorithms in Figure \ref{fig:performance_evaluation}, we will use the best $q$ value observed in Figure \ref{fig:tac_entropic_index} for each benchmark. Besides TAC, we have also examined the influence of entropic index $\eta$ on the performance of RAC and witnessed similar results. Due to space limitation, please refer to Appendix I for more information.
    \begin{figure*}[!ht]
      \begin{minipage}[t]{0.33\textwidth}
        \includegraphics[width=\textwidth]{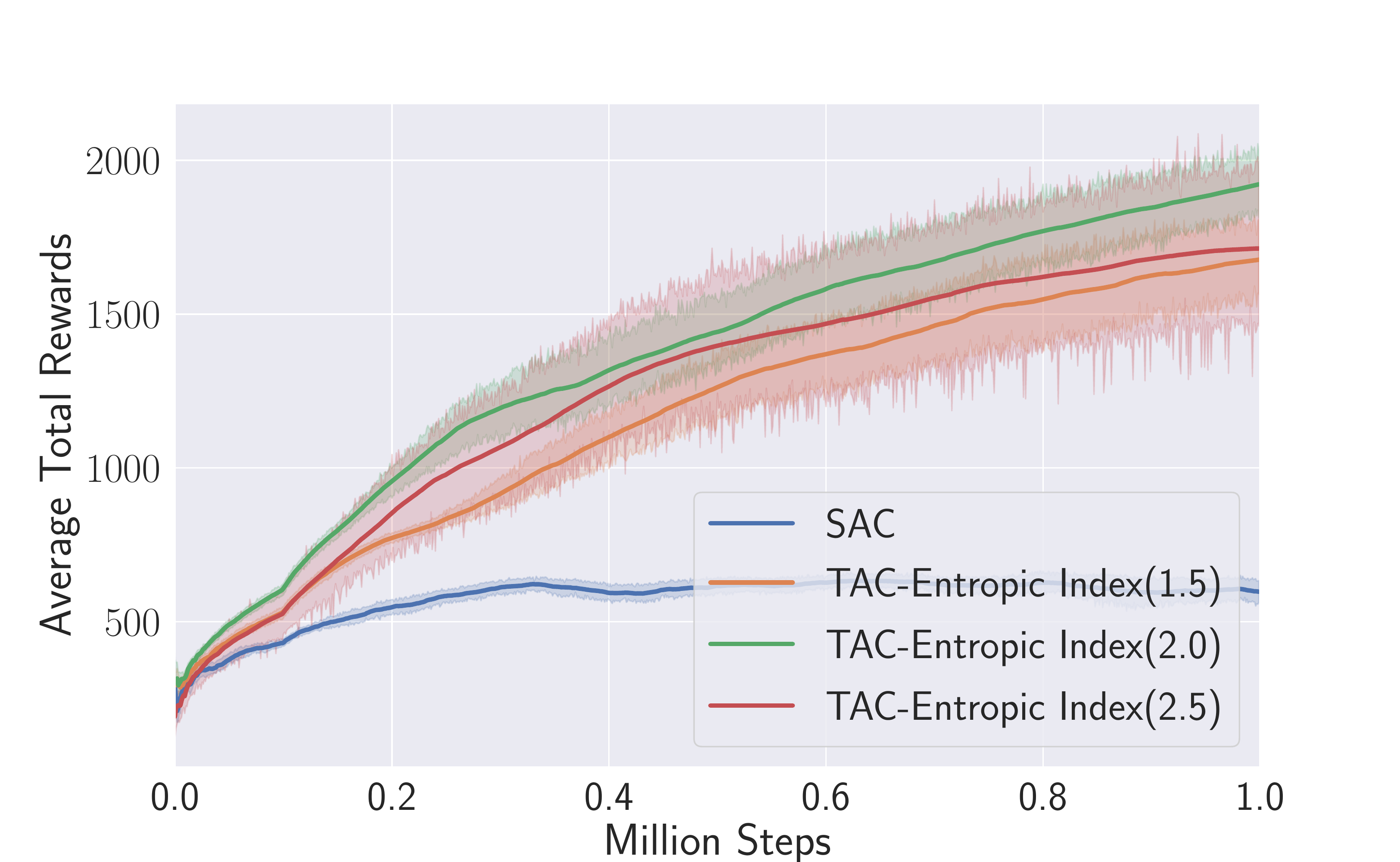}
        \subcaption{\tiny{Ant}}
      \end{minipage}%
      \begin{minipage}[t]{0.33\textwidth}
        \includegraphics[width=\textwidth]{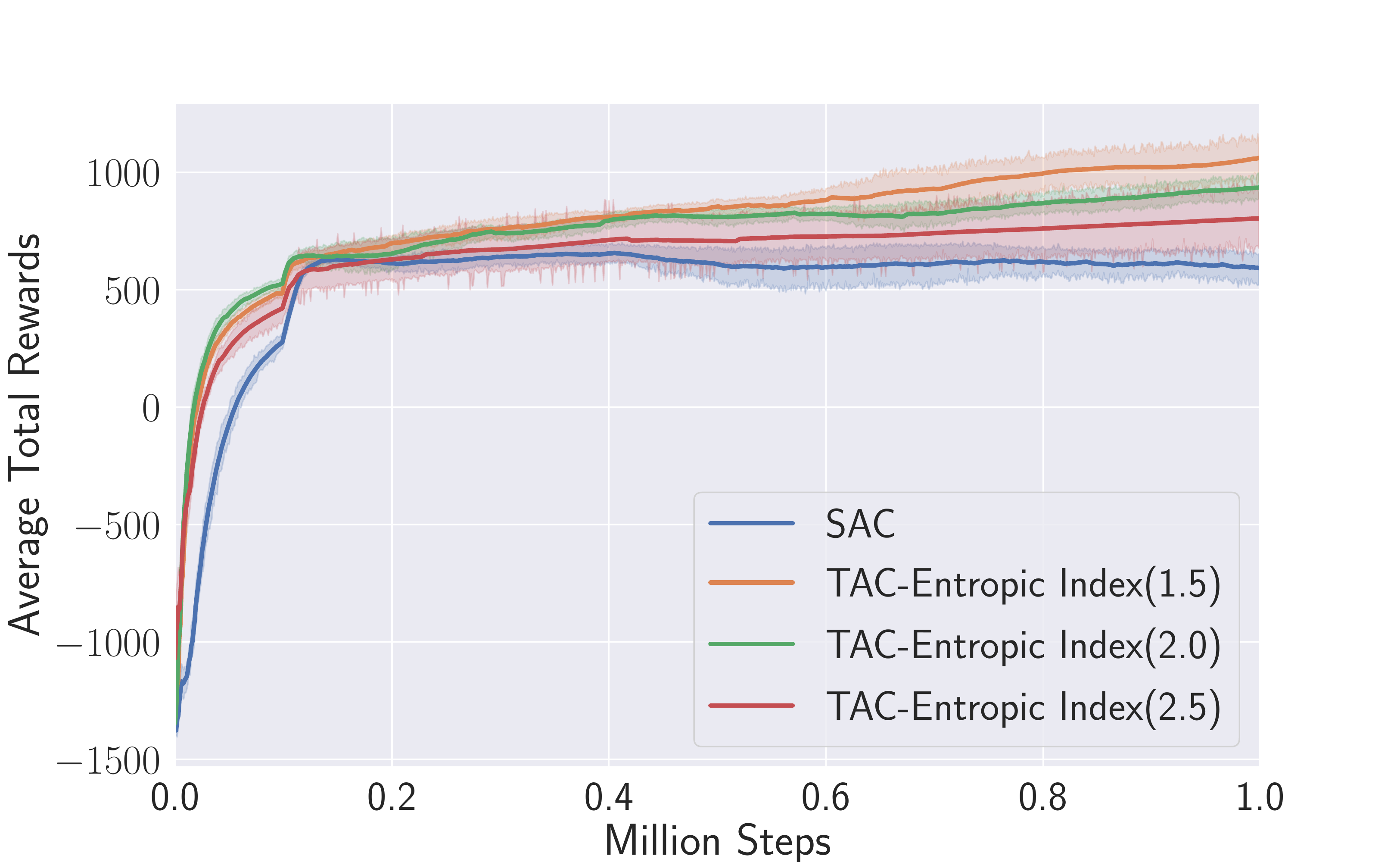}
        \subcaption{\tiny{Half Cheetah}}
      \end{minipage}
      \begin{minipage}[t]{0.33\textwidth}
        \includegraphics[width=\textwidth]{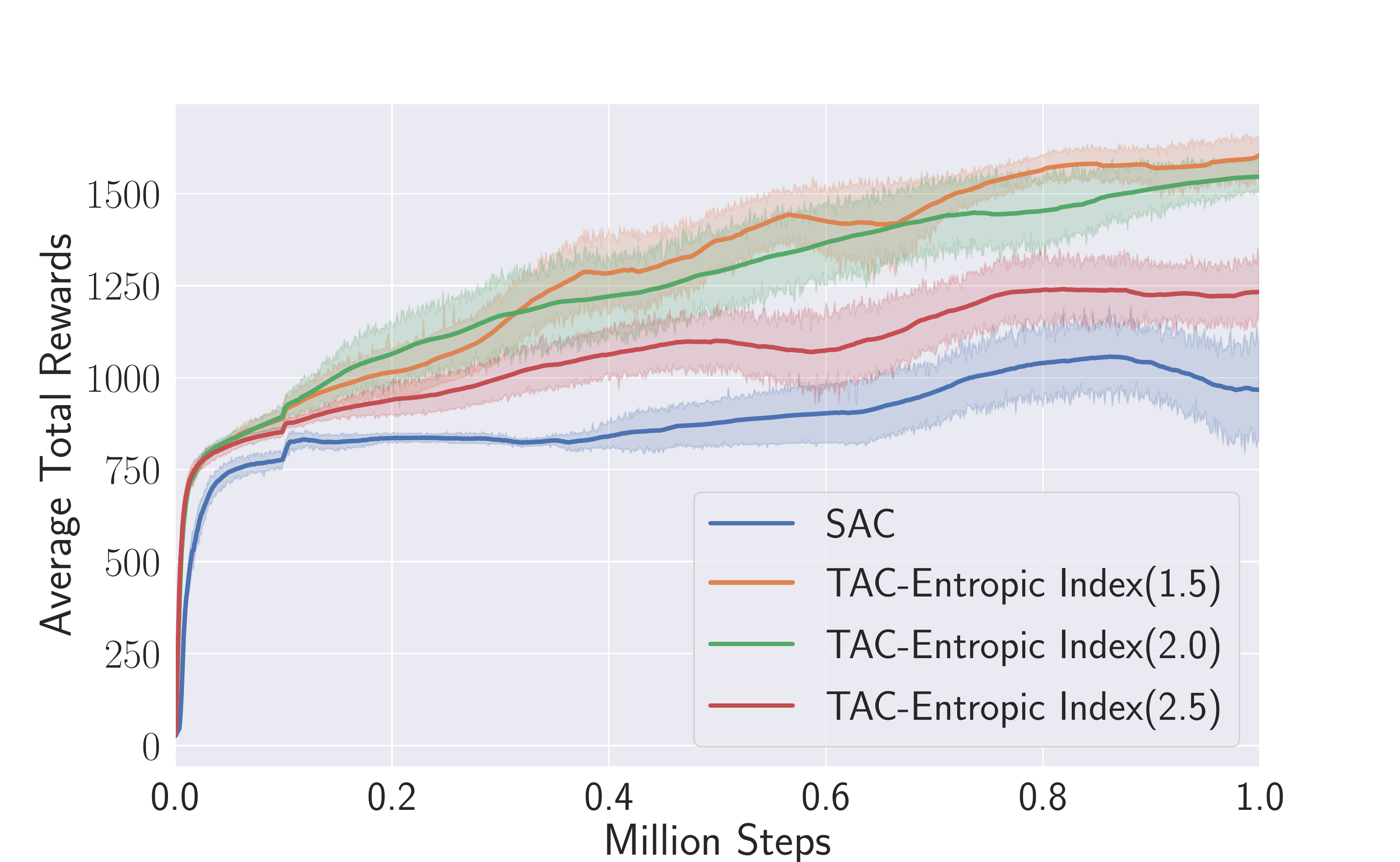}
        \subcaption{\tiny{Hopper}}
      \end{minipage}
      \\
      \begin{minipage}[t]{0.33\textwidth}
        \includegraphics[width=\textwidth]{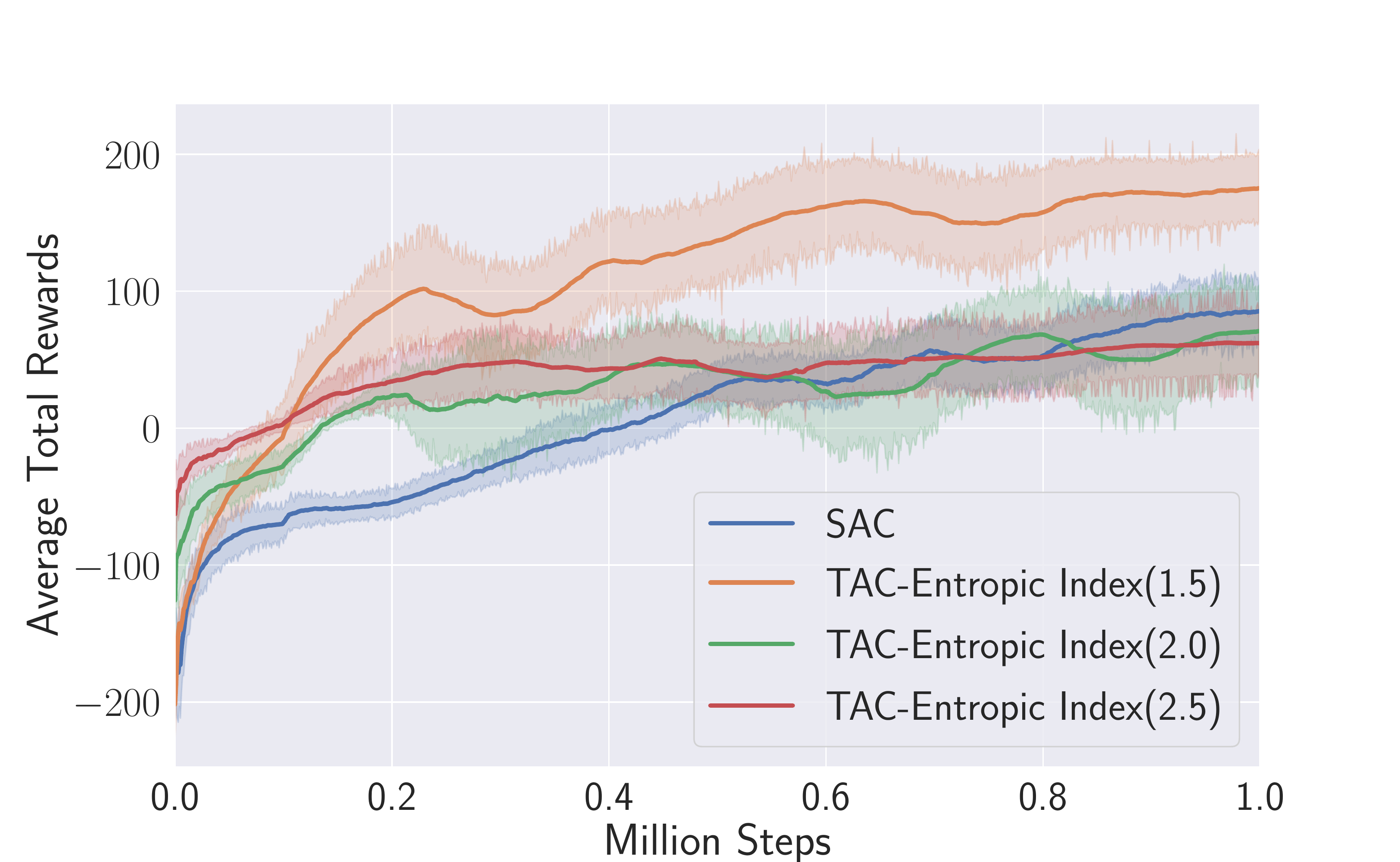}
        \subcaption{\tiny{Lunar Lander}}
      \end{minipage}
      \begin{minipage}[t]{0.33\textwidth}
        \includegraphics[width=\textwidth]{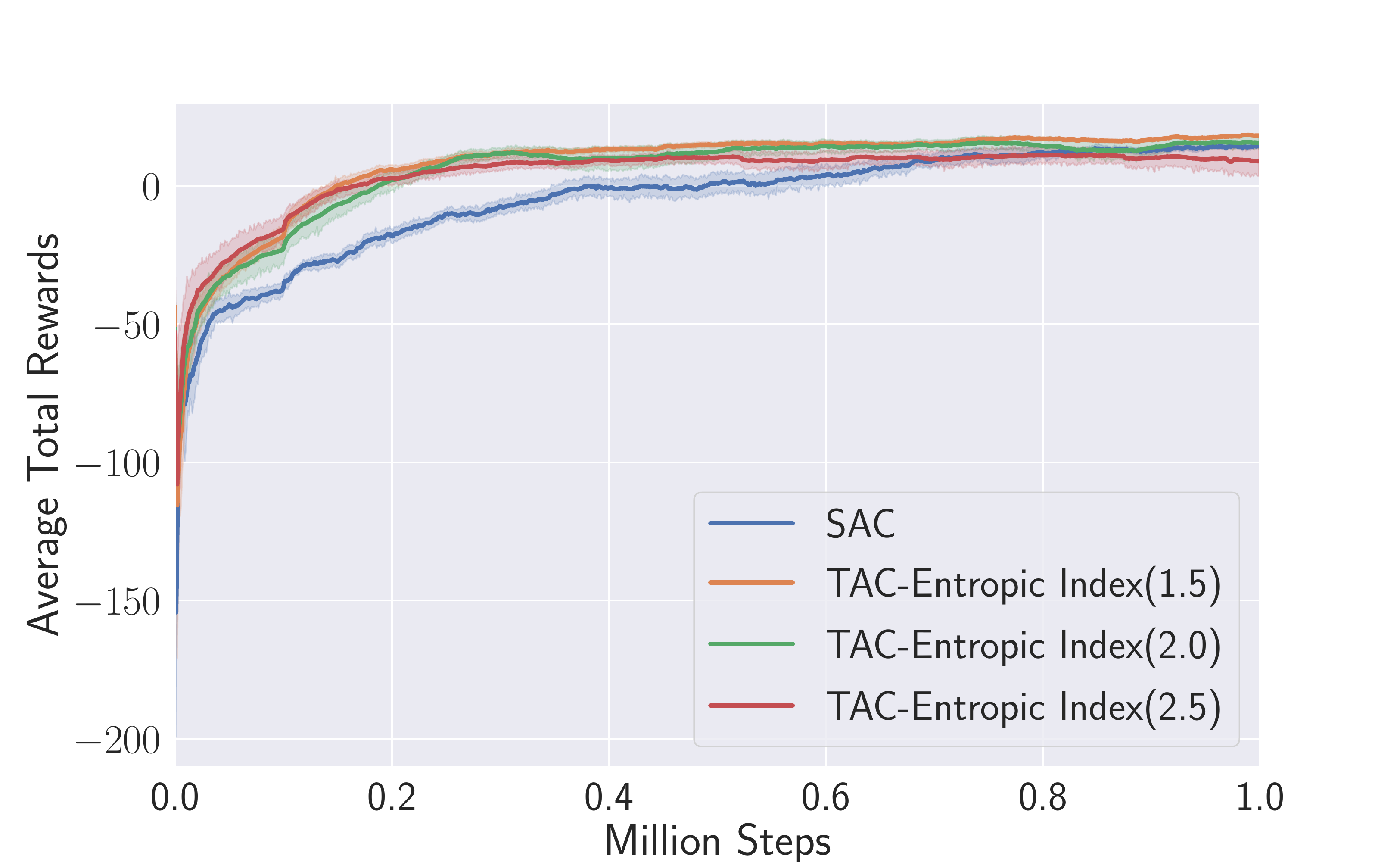}
        \subcaption{\tiny{Reacher}}
      \end{minipage}
      \begin{minipage}[t]{0.33\textwidth}
        \includegraphics[width=\textwidth]{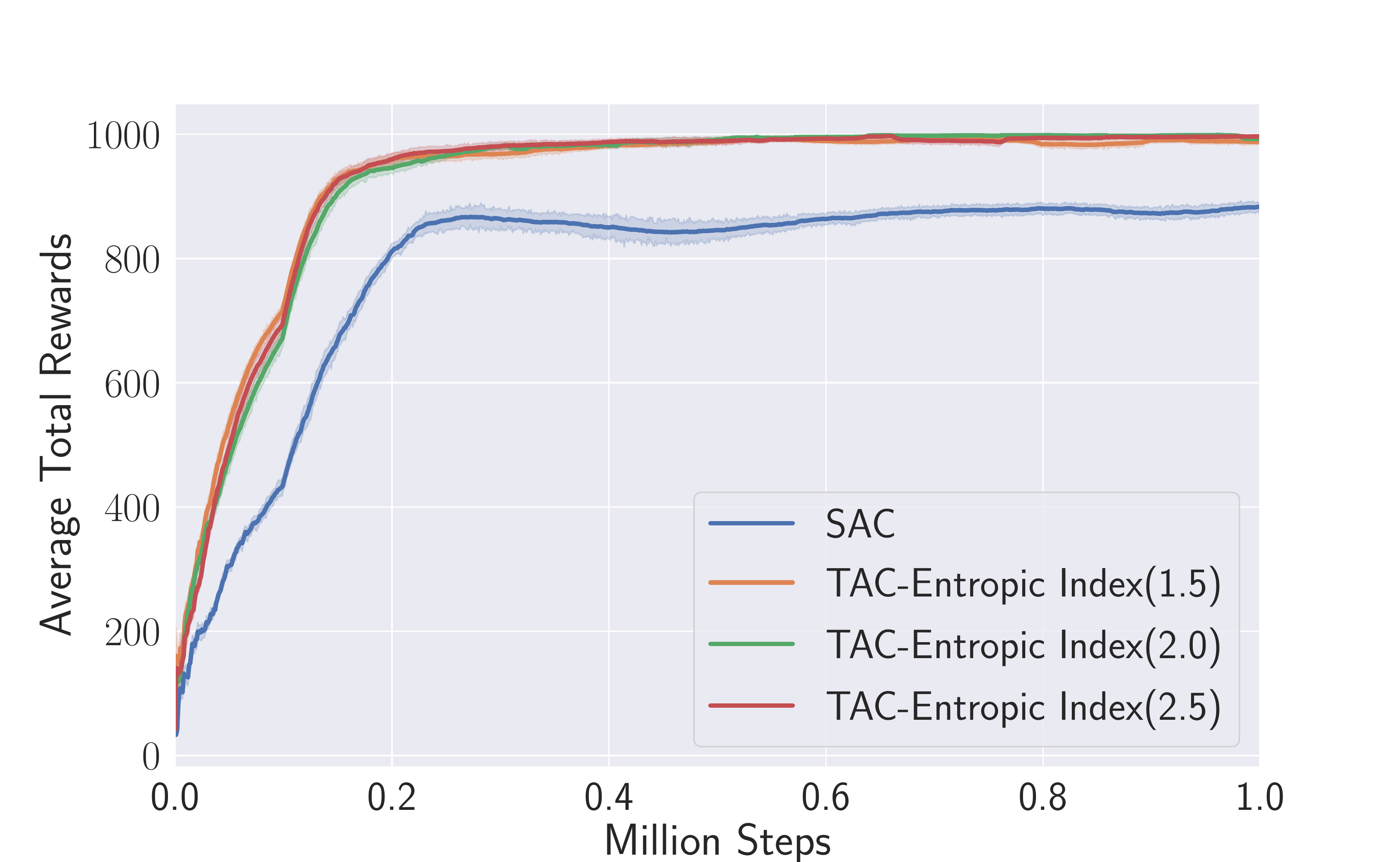}
        \subcaption{\tiny{Walker2D}}
      \end{minipage}
      \caption{The influences of entropic indices ($[1.5, 2.0, 2.5]$) on the performance of TAC on six benchmark control problems, with SAC serving as the baseline.}
      \label{fig:tac_entropic_index}
    \end{figure*}

\label{sub:performance_evaluation}
  \begin{figure*}[!ht]
      \begin{minipage}[t]{0.33\textwidth}
        \includegraphics[width=\textwidth]{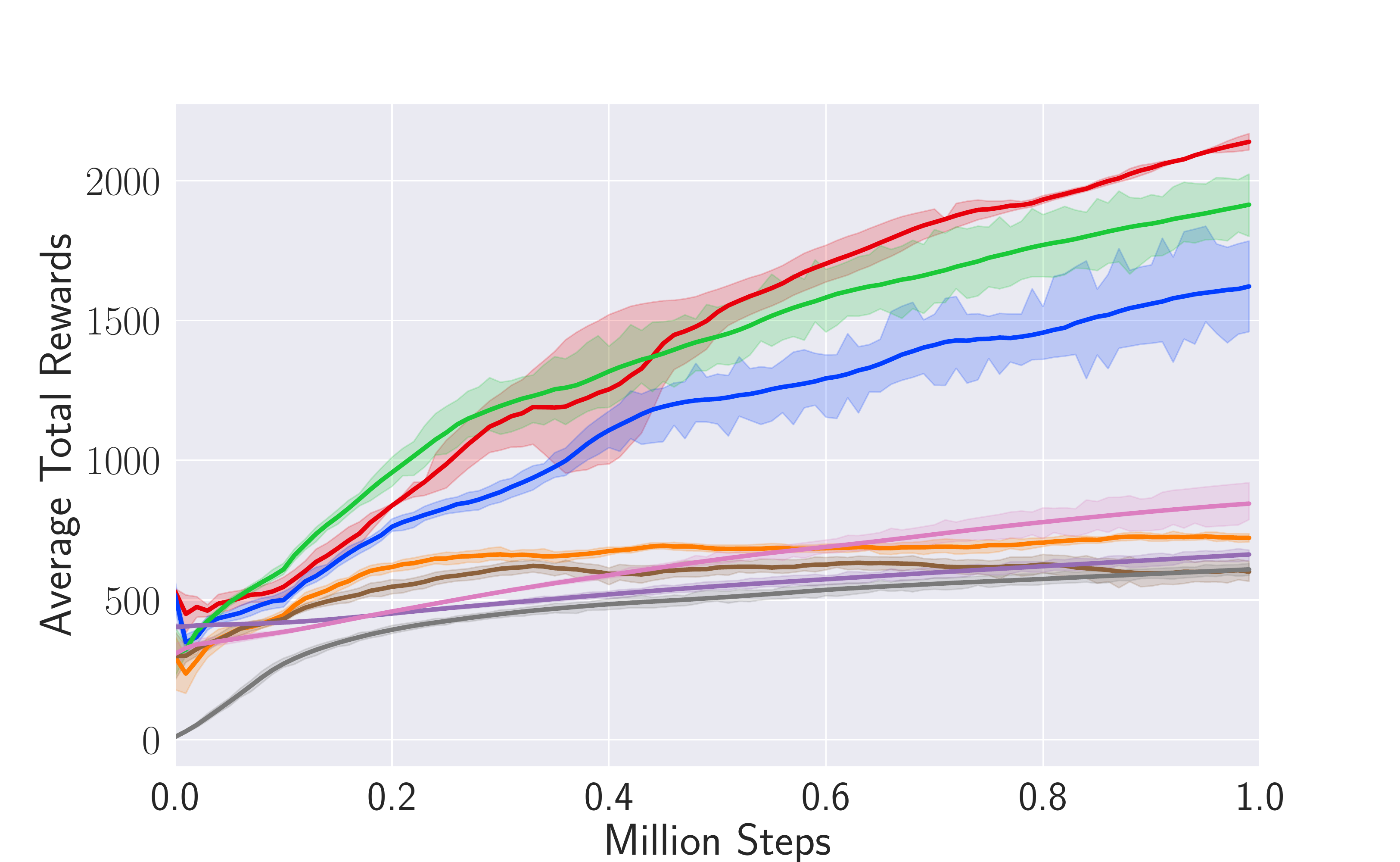}
        \subcaption{\tiny{Ant}}
      \end{minipage}%
      \begin{minipage}[t]{0.33\textwidth}
        \includegraphics[width=\textwidth]{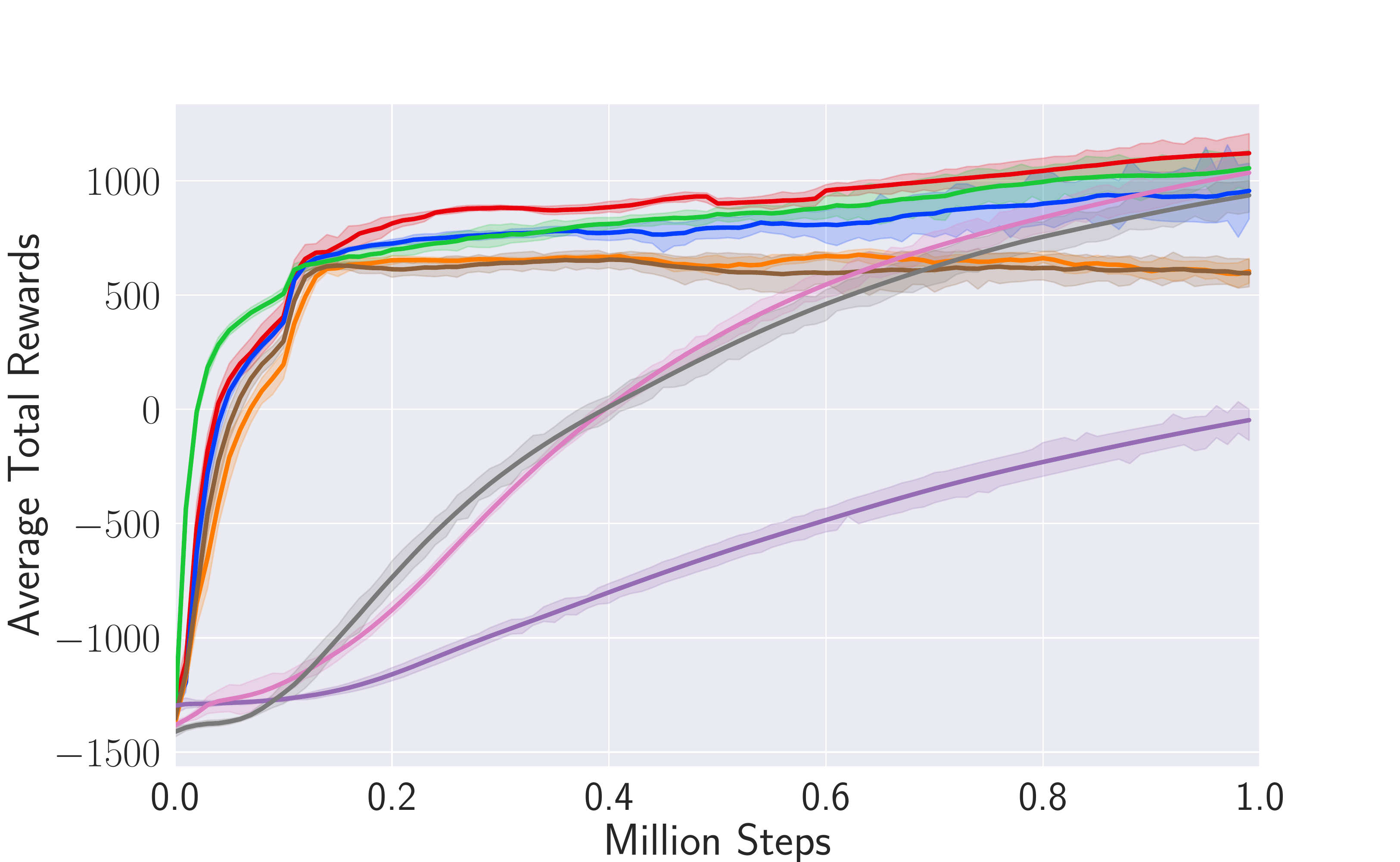}
        \subcaption{\tiny{Half Cheetah}}
      \end{minipage}
      \begin{minipage}[t]{0.33\textwidth}
        \includegraphics[width=\textwidth]{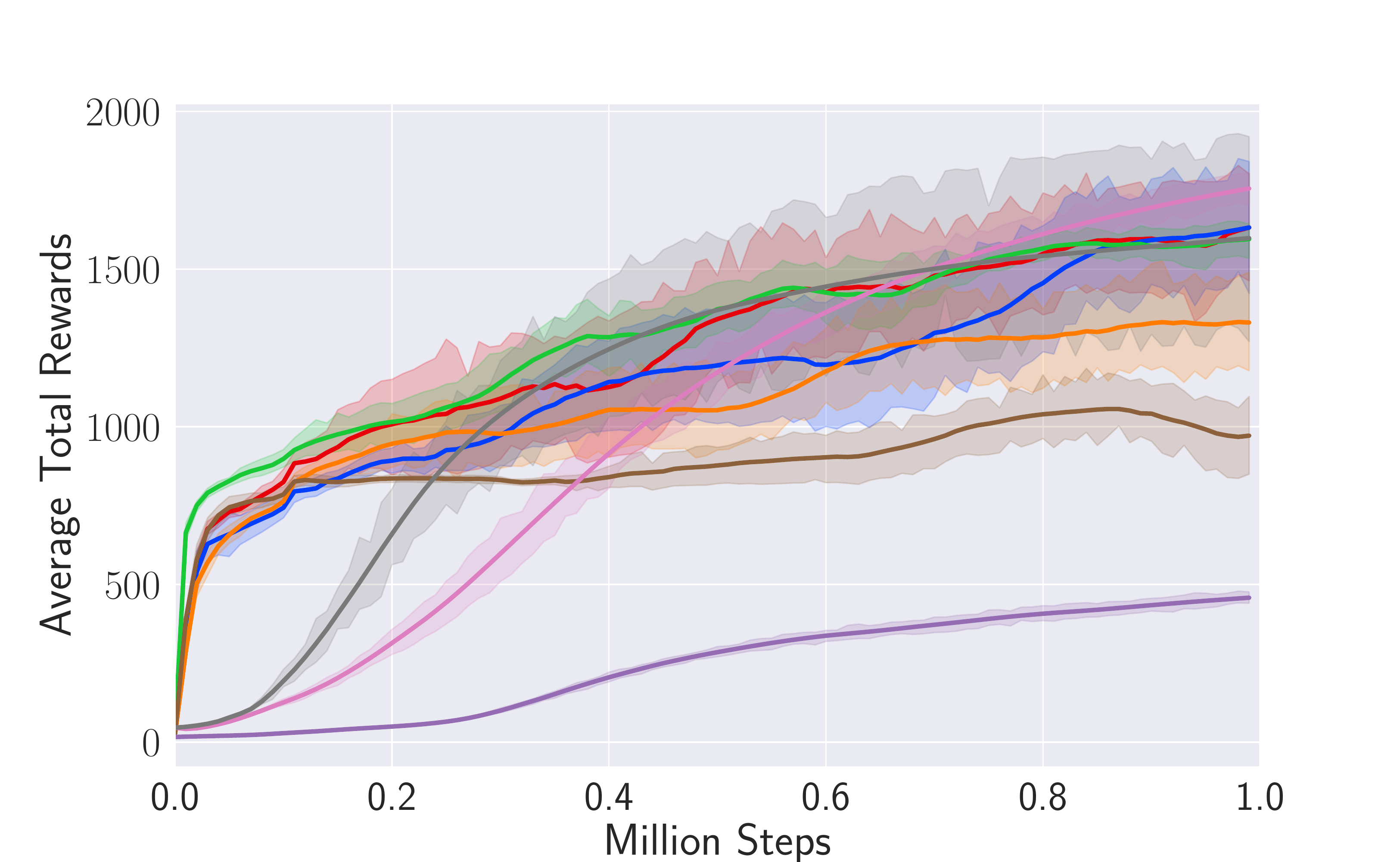}
        \subcaption{\tiny{Hopper}}
      \end{minipage}
      \\
      \begin{minipage}[t]{0.33\textwidth}
        \includegraphics[width=\textwidth]{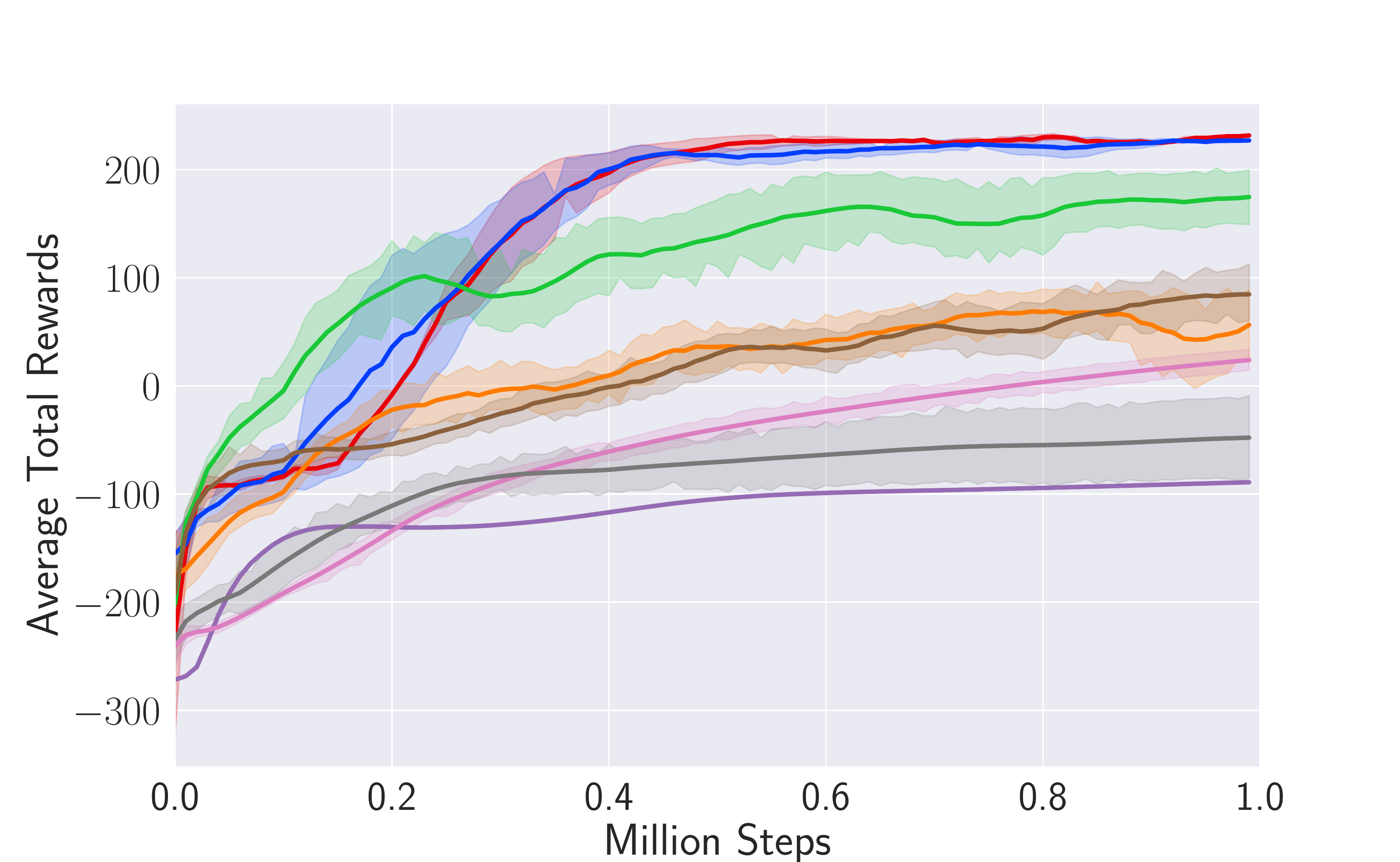}
        \subcaption{\tiny{Lunar Lander}}
      \end{minipage}
      \begin{minipage}[t]{0.33\textwidth}
        \includegraphics[width=\textwidth]{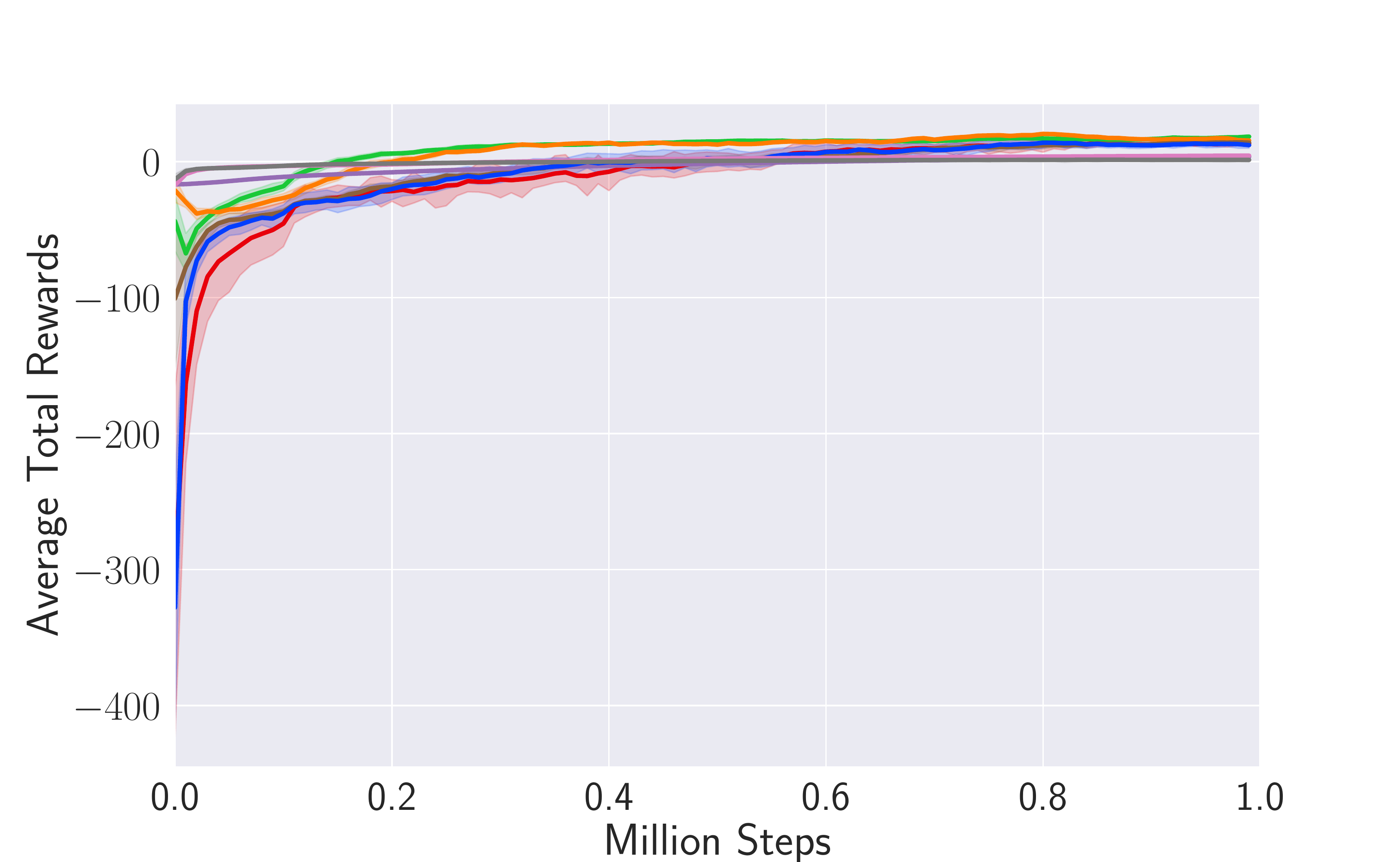}
        \subcaption{\tiny{Reacher}}
      \end{minipage}
      \begin{minipage}[t]{0.33\textwidth}
        \includegraphics[width=\textwidth]{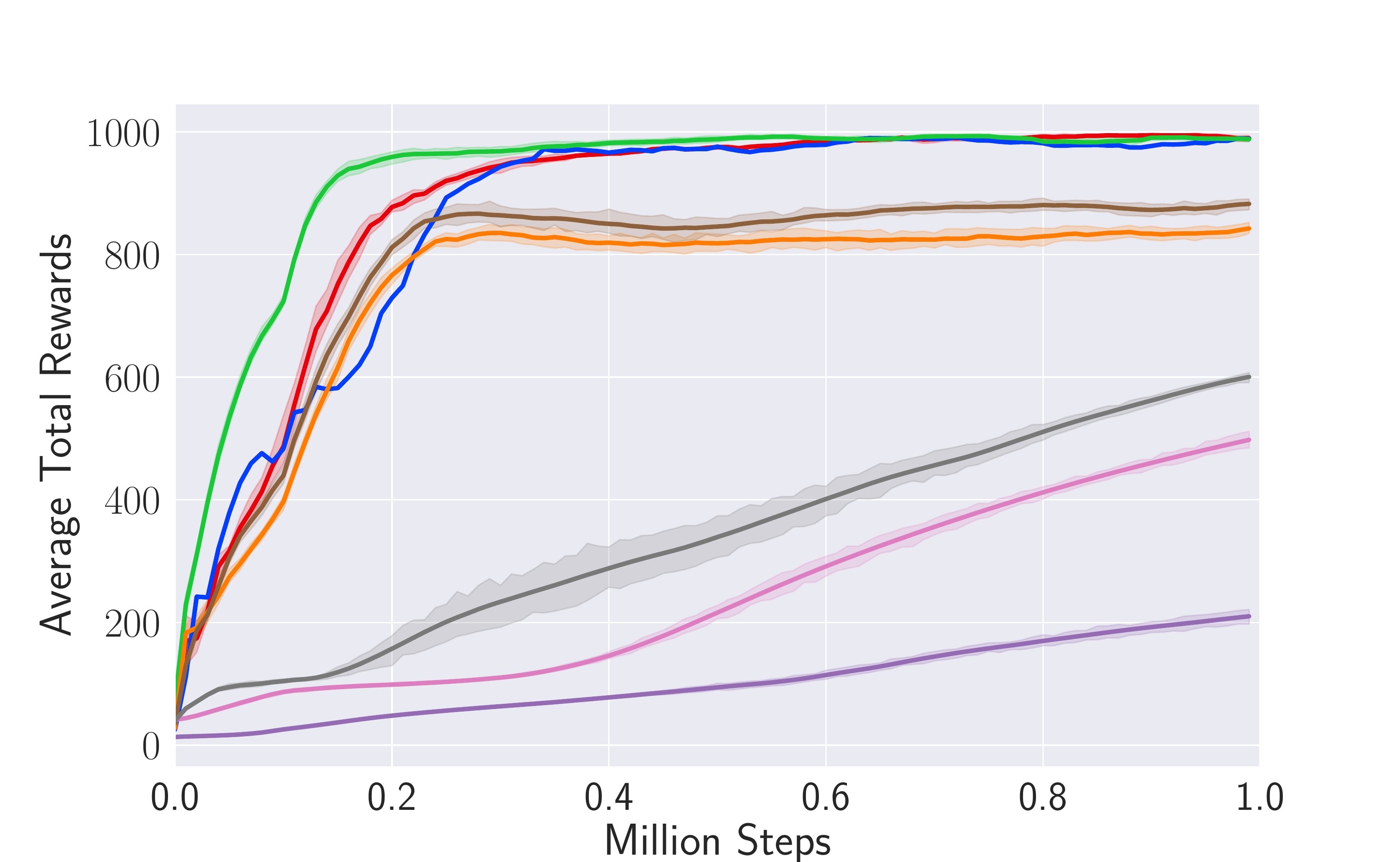}
        \subcaption{\tiny{Walker2D}}
      \end{minipage}
      \\
        \begin{center}
      \begin{minipage}[t]{0.7\textwidth}
        \includegraphics[width=\textwidth]{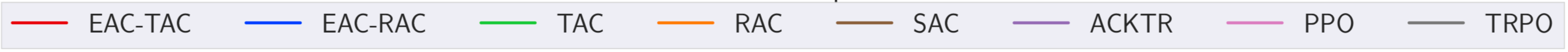}
      \end{minipage}
        \end{center}
      \caption{Performance comparison among all competing algorithms, including SAC, TAC, RAC, EAC-TAC, EAC-RAC, ACKTR, TRPO, and PPO, on six benchmark control problems.}
      \label{fig:performance_evaluation}
    \end{figure*}
 
We next compare the performance of all competing algorithms, as shown in Figure \ref{fig:performance_evaluation}. An inspection of this figure gives rise to two major findings. First, three algorithms proposed by us, i.e., EAC-TAC, EAC-RAC and TAC, have achieved much better performance than other algorithms. In fact, EAC-TAC performed clearly the best on Ant and outperformed all non-ensemble algorithms on Lunar Lander. The three algorithms also achieved the best performance on Walker2D. Moreover, we cannot find any problem on which other algorithms can significantly outperform the three. Second, ensemble techniques can noticeably enhance the reliability and performance of an RL algorithm. For example, EAC-TAC significantly outperformed TAC on two benchmarks  (i.e., Ant and Lunar Lander) and did not perform worse than TAC on the rest. EAC-RAC outperformed RAC on five out of six benchmarks. In general, our experiments show that maximum entropy RL algorithms can be more sample efficient than other competing algorithms. Particularly, during the initial learning phase, these algorithms can learn faster to soon reach a high performance level.

\section{Conclusions}

In this paper, we have established a new policy iteration theory with the aim to maximize arbitrary general entropy measures through Actor-Critic RL. Guided by the theory, we have further developed two new Actor-Critic algorithms, i.e. TAC and RAC, to maximize respectively Tsallis entropy and R\'enyi entropy. Theoretical analysis suggests that these algorithms can be more effective than the recently proposed SAC algorithm. TAC and RAC also inspired us to develop a new ensemble algorithm named EAC that features the use of a bootstrap mechanism for deep environment exploration as well as a new value-function based mechanism for high-level action selection. Empirically we showed that TAC, RAC and EAC can outperform several state-of-the-art learning algorithms. Our research raises an important question regarding the mathematical properties that a general entropy measure must satisfy in order for it to promote effective environment exploration and learning. An answer to this open question will be pursued in the future work.

\small{
\bibliographystyle{named}
\bibliography{citefile}
}

\ \newline
\ \newline
\ \newline
\ \newline
\ \newline
\ \newline
\ \newline
\ \newline
\ \newline
\ \newline
\ \newline
\ \newline
\ \newline
\ \newline
\ \newline

\pagebreak
\vspace{10cm}
\ \ \ 
\vspace{10cm}
\ \ \
\pagebreak

\section*{Appendix A}

Proof of Proposition 1 is presented in this appendix.
\begin{proof}
In order to prove this proposition, it is important to show that the Bellman operator $\mathcal{T}^{\pi}$ for any policy $\pi$, as defined in \eqref{eq-bellman} is a $\gamma$-contraction mapping in the sup-norm. Specifically, at any time $t$, let
\[
\begin{split}
\tilde{r}_{\pi}(s_t,a_t)=&r(s_t,a_t) \\
&+\gamma \alpha \E_{P(s_t,s_{t+1},a_t)} \E_{a_{t+1}\sim\pi(s_{t+1},\cdot)} \mathcal{H}(\pi(s_{t+1},\cdot))
\end{split}
\]
\noindent
then
\[
\begin{split}
\mathcal{T}^{\pi}Q(s_t,a_t)& =\tilde{r}_{\pi}(s_t,a_t) \\
&+\gamma\E_{P(s_t,s_{t+1},a_t)}\E_{a_{t+1}\sim\pi(s_{t+1},\cdot)}Q(s_{t+1},a_{t+1})
\end{split}
\]
Given any two distinct Q-functions $Q$ and $Q'$, based on the definition of $\mathcal{T}^{\pi}$ above, it is straightforward to verify that
\[
\|Q-Q'\|_{\infty} \geq \gamma \|\mathcal{T}^{\pi} Q-\mathcal{T}^{\pi}Q' \|
\]
\noindent
Therefore $\mathcal{T}^{\pi}$ is a $\gamma$-contraction mapping. According to the basic convergence theory for policy evaluation \cite{sutton1998book}, it is immediate to conclude that repeated application of $\mathcal{T}^{\pi}$ on any Q-function $Q^0$ will converge to $Q^{\pi}$ for arbitrary policy $\pi$.
\end{proof}

\section*{Appendix B}

Proof of Proposition 2 is presented in this appendix.
\begin{proof}
Since policy $\pi'$ solves the maximization problem formulated in \eqref{eq-pi-improve-gac} with respect to policy $\pi$, therefore for any state $s\in\mathbb{S}$ and any action $a\in\mathbb{A}$,
\begin{equation}
\begin{split}
\E_{a\sim\pi'(s,\cdot)} \left( Q^{\pi}(s,a)+\alpha \mathcal{H}(\pi')\right) &\geq \E_{a\sim\pi(s,\cdot)} \left(Q^{\pi}(s,a)+\alpha \mathcal{H}(\pi) \right) \\
&= V^{\pi}(s)
\end{split}
\label{eq-bd-propb}
\end{equation}
\noindent
Based on this inequality, at any time $t$, consider the following Bellman equation
\[
\begin{split}
& Q^{\pi}(s_t,a_t)= r(s_t,a_t)+\gamma\E_{P(s_t,s_{t+1},a_t)} V^{\pi}(s_{t+1})\\
& \leq r(s_t,a_t)+\gamma\E_{P(s_t,s_{t+1},a_t)} \E_{a_{t+1}\sim\pi'(s_{t+1},\cdot)}[
Q^{\pi}(s_{t+1},a_{t+1})+ \\ & \ \ \ \ \ \alpha\mathcal{H}(\pi'(s_{t+1},\cdot))
]\\
&=\tilde{r}_{\pi'}+\gamma\E_{P(s_t,s_{t+1},a_t)}\E_{a_{t+1}\sim\pi'(s_{t+1},\cdot)}Q^{\pi}(s_{t+1},a_{t+1}) \\
&=\tilde{r}_{\pi'}+\gamma \E_{P(s_t,s_{t+1},a_t)}\E_{a_{t+1}\sim\pi'(s_{t+1},\cdot)} [ r(s_{t+1},a_{t+1})+\\ &\ \ \ \ \ \gamma\E_{P(s_{t+1},s_{t+2},a_{t+1})} V^{\pi}(s_{t+2}) ] \\
&\leq \tilde{r}_{\pi'}(s_t,a_t) + \gamma \E_{P(s_t,s_{t+1},a_t)}\E_{a_{t+1}\sim\pi'(s_{t+1},\cdot)} \tilde{r}_{\pi'}(s_{t+1},a_{t+1})\\
& \ \ \ \ \ + \gamma^2 \E_{P(s_t,s_{t+1},a_t)}\E_{a_{t+1}\sim\pi'(s_{t+1},\cdot)}\E_{P(s_{t+1},s_{t+2},a_{t+1})}\\
& \ \ \ \ \ \E_{a_{t+2}\sim\pi'(s_{t+2},\cdot)} Q^{\pi}(s_{t+2},a_{t+2}) \\
& \vdots \\
& \leq Q^{\pi'}(s_t,a_t)
\end{split}
\]
\noindent
The inequality above is realized through repeated expanding of $Q^{\pi}$ based on the Bellman equation and \eqref{eq-bd-propb}. The last line of the inequality is derived from Proposition \ref{proposition-1}. It can be concluded now that policy $\pi'$ is an improvement over policy $\pi$. In other words, the policy improvement step governed by \eqref{eq-pi-improve-gac} is effective.
\end{proof}

\section*{Appendix C}

Proof of Proposition 3 is presented in this appendix.
\begin{proof}
It is not difficult to see that repeated application of the policy improvement mechanism defined in \eqref{eq-pi-improve-gac} enables us to build a sequence of policies $\pi_1,\pi_2,\ldots,\pi_i,\ldots$. Moreover, due to Proposition \ref{proposition-2}, the policies created in the sequence are monotonically increasing in performance (in terms of Q-function). For every $\pi_i\in\Pi$ where $i=1,2,\ldots$, $Q^{\pi_i}$ is bounded from above since both the step-wise reward and the entropy of $\pi_i$ are assumed to be bounded. In view of this, the sequence must converge to a specific policy $\pi^*\in\Pi$. Now consider another policy $\pi\in\Pi$ such that $\pi\neq\pi^*$. We can see that, for any state $s\in\mathbb{S}$ and any action $a\in\mathbb{A}$,
\[
\E_{a\sim\pi^*(s,\cdot)}Q^{\pi^*}(s,a)+\alpha\mathcal{H}(\pi^*)\geq \E_{a\sim\pi(s,\cdot)} Q^{\pi^*}(s,a)+\alpha\mathcal{H}(\pi)
\]
\noindent
Hence
\[
V^{\pi^*}(s)\geq \E_{a\sim\pi(s,\cdot)} Q^{\pi^*}(s,a)+\alpha\mathcal{H}(\pi)
\]
\noindent
In line with this inequality, at any time $t$,
\[
\begin{split}
& Q^{\pi^*}(s_t,a_t)=r(s_t,a_t)+\gamma\E_{P(s_t,s_{t+1},a_t)} V^{\pi^*}(s_{t+1}) \\
& \geq r(s_t,a_t)+\gamma\E_{P(s_t,s_{t+1},a_t)} \E_{a_{t+1}\sim\pi(s_{t+1},\cdot)} Q^{\pi^*}(s_{t+1},a_{t+1}) \\
& \ \ \ \ \ +\alpha\mathcal{H}(\pi(s_{t+1},\cdot)) \\
&\vdots \\
& \geq Q^{\pi}(s_t,a_t)
\end{split}
\]
\noindent
Consequently, it is impossible to find another policy $\pi\in\Pi$ that performs better than policy $\pi^*$. Therefore the policy iteration algorithm must converge to an optimal stochastic policy $\pi^*$ where the optimality is defined in terms of the Q-function in \eqref{eq-q-func}.
\end{proof}

\section*{Appendix D}

This appendix presents proof of Corollary \ref{corollary-1}.
\begin{proof}
Consider the procedure of building a new and better policy $\pi'$ from an existing policy $\pi$. By using SAC's policy improvement mechanism in \eqref{eq-pi-improve-sac}, for any state $s\in\mathbb{S}$, $\pi'$ is expected to minimize the KL divergence below.
\[
\begin{split}
& D_{KL}(\pi'(s,\cdot),\exp(Q^{\pi}(s,\cdot)-\mathbb{C}_s)) \\
= &\int_{a\in\mathbb{A}} \left(\pi'(s,a)\log\pi'(s,a)-\pi'(s,a)Q^{\pi}(s,a)\right)\mathrm{d}a+\mathbb{C}_s \\
= & -\E_{a\sim\pi'(s,\cdot)} \left[ Q^{\pi}(s,a)+\alpha\mathcal{H}^s(\pi'(s,\cdot)) \right] + \mathbb{C}_s
\end{split}
\]
\noindent
where $\mathcal{H}^s$ stands for Shannon entropy. Because $\mathbb{C}_s$ remains as a constant in any specific state $s$ (i.e. $\mathbb{C}_s$ is a function of state $s$, but not a function of action $a$), we can ignore $\mathbb{C}_s$ while minimizing the KL divergence in state $s$. In other words, minimizing $D_{KL}$ is equivalent to solving the maximization problem in \eqref{eq-pi-improve-gac} whenever $\mathcal{H}(\pi')=\mathcal{H}^s(\pi')$ and $\alpha=1$. Therefore, for two policies $\pi'$ and $\pi"$ that solve respectively the optimization problems in \eqref{eq-pi-improve-sac} and \eqref{eq-pi-improve-gac}, $Q^{\pi'}(s,a)=Q^{\pi"}(s,a)$ for any $s\in\mathbb{S}$ and any $a\in\mathbb{A}$.
\end{proof}

\section*{Appendix E}

This appendix details the methods we used to evaluate $\int \pi^{\eta}(s,a)\mathrm{d}a$ and $\mathcal{H}^{\eta}(\pi)$ in any state $s\in\mathbb{S}$. As we mentioned in the paper, $\int \pi^{\eta}(s,a)\mathrm{d}a$ can be either computed exactly or approximated through a group of sampled actions. This depends on the type of probability distributions for action selection supported by the corresponding policy family $\Pi$. Specifically, when $\pi(s,\cdot)$ represents an $m$-dimensional multivariate normal distribution with diagonal covariance matrix $\Sigma_{\pi}$, $\int \pi^{\eta}(s,a)\mathrm{d}a$ can be determined easily and precisely as
\[
\int_{a\in\mathbb{A}}\pi^{\eta}(s,a)\mathrm{d}a=\left( \sqrt{2\pi} \right)^{m (1-\eta)}\eta^{-\frac{m}{2}}\prod_{j=1}^m \sigma_j^{1-\eta}
\]
\noindent
where $\sigma_j$ with $j=1,\ldots,m$ refers to the square root of each of the $m$ diagonal elements of matrix $\Sigma_{\pi}$. Subsequently, R\'enyi entropy of policy $\pi$ can be obtained directly as
\[
\mathcal{H}^{\eta}(\pi(s,\cdot))=\frac{m}{2}\log 2\pi-\frac{m}{2(1-\eta)}\log\eta +\sum_{j=1}^m \log\sigma_j
\]

However in our experiments, following the \emph{squashing technique} introduced in \cite{haarnoja2018icml}, the internal action $u$ will be sampled initially from the multivariate normal distribution denoted as $\mu(s,u)$ with mean $\bar{\mu}(s)$ and diagonal covariance matrix $\Sigma_{\mu}(s)$. The sampled internal action will be subsequently passed to an invertible squash function $tanh$ to produce the output action $a\in [-1,1]^m$, i.e. $a=tanh(u)$. In line with this squashing technique,
\[
\begin{split}
& \int_{a\in\mathbb{A}}\pi(s,a)^{\eta}\mathrm{d}a =\\
& \int_{u\in \mathbb{R}^m}\left(
\mu(s,u)\left| \nabla_u tanh(u) \right|^{-1}
\right)^{\eta} \nabla_u tanh(u) \mathrm{d}u
\end{split}
\]
\noindent
Unfortunately, fixed-form solution of the integral above does not exist for arbitrary settings of $\eta$. We therefore decided to approximate it efficiently through a random sampling technique. In particular,
\[
\int_{a\in\mathbb{A}}\pi(s,a)^{\eta}\mathrm{d}a=\E_{u\sim\mu(s,\cdot)}\left( \mu(s,u)^{\eta-1} \left(\nabla_u tanh(u)\right)^{1-\eta} \right)
\]
Accordingly, assuming that $u_1,\ldots,u_k$ are $k$ internal actions sampled independently from $\mu(s,\cdot)$, then
\[
\int_{a\in\mathbb{A}}\pi(s,a)^{\eta}\mathrm{d}a\approx\frac{1}{k}\sum_{i=1}^k \mu(s,u_i)^{\eta-1}\left( \nabla_u tanh(u)\right)^{1-\eta}|_{u=u_i}
\]
Subsequently $\mathcal{H}^{\eta}(\pi)$ can be approximated straightforwardly based on its definition. Apparently, with more sampled actions, the approximation will become more precise. However increasing the number of sampled actions will inevitably prolong the learning time required. In practice we found that using 9 randomly sampled actions can produce reasonably good learning performance without sacrificing noticeably on computation time.

\section*{Appendix F}

In this appendix, we aim to develop and analyze the performance lower bounds of TAC and RAC, in comparison to standard RL algorithms that maximize cumulative rewards alone (without considering the maximum entropy objective). Under the conventional learning framework in \eqref{eq-lt-cum-rew}, the Q-function can be updated through the \emph{standard Bellman operator}, as defined below
\[
\mathcal{T}Q(s,a)=r(s,a)+\gamma\int_{s'\in\mathbb{S}} P(s,s',a)\max_{a'\in\mathbb{A}} Q(s',a') \mathrm{d}s
\]
\noindent
for any state $s\in\mathbb{S}$ and any action $a\in\mathbb{A}$. Different from this approach, in association with the maximum entropy learning framework presented in \eqref{eq-lt-cum-rew-ext}, the Q-function will be updated via the following \emph{maximum entropy Bellman operator}
\begin{equation}
\begin{split}
\mathcal{T}_{\mathcal{H}}Q(s,a)& =r(s,a)+\\
& \gamma\int_{s'\in\mathbb{S}} P(s,s',a) \int_{a\in\mathbb{A}}\pi_{\mathcal{H}}(s',a') Q(s',a') \mathrm{d}a \mathrm{d}s
\end{split}
\label{eq-er-bell-opt}
\end{equation}
\noindent
where $\pi_{\mathcal{H}}\in\Pi$ stands for the stochastic policy obtained by solving the policy improvement problem in \eqref{eq-pi-improve-gac} for both TAC and RAC. In order to analyze $\pi_{\mathcal{H}}$ theoretically, we make several key assumptions as summarized below. Our reliance on these assumptions prevents our analysis from being generally applicable. However, the performance lower bounds derived from our analysis still shed new light on the practical usefulness of TAC and RAC.

\begin{itemize}
\item[\textbf{A1}] The action space of the RL problem is unbounded and 1-dimensional $(-\infty,\infty)$.
\item[\textbf{A2}] The Q-function in any state $s\in\mathbb{S}$ is a bell-shaped non-negative function as defined below
$$
Q(s,a)=\zeta_s \exp\left( -\frac{(a-\bar{a}_s)^2}{2 \xi_s^2} \right)
$$
\noindent
where $\zeta_s,\xi_s > 0$ and both are bounded from above and below across all states. Particularly,
$$
\zeta^* = \max_{s\in\mathbb{S}}\max_{a\in\mathbb{A}} Q(s,a), \xi^*=\max_{s\in\mathbb{S}}\xi_s
$$
\noindent
and
$$
\zeta_* = \min_{s\in\mathbb{S}}\max_{a\in\mathbb{A}} Q(s,a), \xi_*=\min_{s\in\mathbb{S}}\xi_s
$$
\item[\textbf{A3}] An RL agent that follows any policy $\pi\in\Pi$ will select actions in any state $s\in\mathbb{S}$ according to a normal distribution determined by $\pi(s,\cdot)$ with mean $\bar{a}_{\pi,s}$ and standard deviation $\sigma_{\pi,s}\leq \sigma^*$.
\end{itemize}

Assumption A1 is not essential. However its use simplifies our analysis to be presented below. It is possible for us to extend the action space of an RL problem to multiple dimensions but we will not pursue this direction further in this appendix. Assumption A2 can be interpreted in two different ways. Specifically, we can consider A2 as a reflection of the modelling restriction on the Q-function, due to which the estimated Q-function always assume a bell-shaped curve in any state $s$. Similar bell-shaped Q-functions have been utilized in \cite{gu2018icml}. Alternatively, we may assume that the improved policy $\pi_{\mathcal{H}}$ derived from A2 can closely approximate the performance of $\pi_{\mathcal{H}}$ in the case when A2 does not hold consistently. Finally assumption A3 is satisfied by many existing RL algorithms when applied to benchmark control tasks where the actions to be performed in any state during learning are sampled from normal distributions \cite{schulman20171,wu2017nips}.

Following the three assumptions above, we can establish two lemmas below with regard to the cases when Tsallis entropy and R\'enyi entropy are utilized respectively for maximum entropy RL.
\begin{lemma}
Under the assumptions of A1, A2 and A3, when Tsallis entropy $\mathcal{H}^q$ is adopted by the policy improvement mechanism in \eqref{eq-pi-improve-gac}, then
\[
\begin{split}
& \max_{a\in\mathbb{A}}Q(s,a)-\E_{a\sim\pi_{\mathcal{H}^q}} Q(s,a) \leq \zeta^* \max_{\xi\in [\xi_*,\xi^*]} \\
& \left(1- \frac{\xi}{\sqrt{
\xi^2+  \min\left\{ \left(\frac{ \alpha (2\pi)^{ \frac{1-q}{2} } (\sigma^{*2}+\xi^2)^{\frac{3}{2}} }{\sqrt{q}\xi\zeta_*} \right)^{\frac{2}{1+q}},\sigma^{*2} \right\}
}} \right)\\
& =\zeta_{\mathcal{H}_q}
\end{split}
\]
\label{lemma-1}
\end{lemma}
\begin{lemma}
Under the assumptions of A1, A2 and A3, when R\'enyi entropy $\mathcal{H}^{\eta}$ is adopted by the policy improvement mechanism in \eqref{eq-pi-improve-gac}, then
\[
\begin{split}
& \max_{a\in\mathbb{A}}Q(s,a)-\E_{a\sim\pi_{\mathcal{H}^{\eta}}} Q(s,a) \\ & \leq \zeta^*\max_{\xi\in [\xi_*,\xi^*]}\left(1-\frac{ \xi }{ \sqrt{
\xi^2 + \min\left\{ \frac{ \alpha (\sigma^{*2}+\xi^2)^{\frac{3}{2}} }{ \xi\zeta_* },\sigma^{*2} \right\}
} }
\right)\\
& =\zeta_{\mathcal{H}_{\eta}}
\end{split}
\]
\label{lemma-2}
\end{lemma}

\noindent
{\bf Proof of Lemma \ref{lemma-1}}:

\begin{proof}
To prove this lemma, we must determine policy $\pi_{\mathcal{H}^q}$. Due to A2 and A3, we know that $\bar{a}_{\pi,s}=\bar{a}_s$ for $\pi_{\mathcal{H}^q}$. For any policy $\pi$ that satisfies this condition in any $s\in\mathbb{S}$, we have
\[
\begin{split}
& \E_{a\in\pi(s,\cdot)} Q(s,a)+\alpha\mathcal{H}^q(\pi(s,\cdot))=\\
& \frac{\xi_s\zeta_s}{\sqrt{\xi_s^2+\sigma_{\pi,s}^2}}+\alpha
\frac{ (2\pi)^{\frac{1-q}{2}} \sigma_{\pi,s}^{1-q} -\sqrt{q} }{\sqrt{q}(1-q)}
\end{split}
\]
\noindent
In order to maximize the expectation above through adjusting $\sigma_{\pi,s}$, we take the derivative of the expectation with respect to $\sigma_{\pi,s}$ and obtain the equation below.
\begin{equation}
\frac{ \alpha (2\pi)^{\frac{1-q}{2}} \sigma_{\pi,s}^{-q} }{ \sqrt{q} } - \frac{\xi_s\zeta_s \sigma_{\pi,s}}{ (\xi_s^2+\sigma_{\pi,s}^2)^{\frac{3}{2}} }=0
\label{eq-max-sigma-tsallis}
\end{equation}
As a result of solving this equation, we can obtain $\sigma_{\pi_{\mathcal{H}^q}}$ for policy $\pi_{\mathcal{H}^q}$. Unfortunately analytic solution of this equation does not exist for arbitrary settings of $q>1$, $\xi>0$, $\zeta_s>0$ and $\alpha>0$. Nevertheless, let $\bar{\sigma}$ be the solution of the equation below
\[
\frac{ \alpha (2\pi)^{\frac{1-q}{2}} \sigma_{\pi,s}^{-q} }{ \sqrt{q} } - \frac{\xi_s\zeta_s \sigma_{\pi,s}}{ (\xi_s^2+\sigma^{*2})^{\frac{3}{2}} }=0
\]
It can be easily verified that $\sigma_{\pi_{\mathcal{H}^q}}\leq \bar{\sigma}$ if the solution of \eqref{eq-max-sigma-tsallis} is less than or equal to $\sigma^*$. On the other hand, if the solution of \eqref{eq-max-sigma-tsallis} is greater than $\sigma^*$, then $\sigma_{\pi_{\mathcal{H}^q}}=\sigma^*$. In line with this understanding, the inequality below can be derived.
\[
\sigma_{\pi_{\mathcal{H}^q}}\leq \min\left\{ \left(
\frac{ \alpha (2\pi)^{\frac{1-q}{2}} (\sigma^{*2}+\xi_s^2)^{\frac{3}{2}} }{ \sqrt{q}\xi_s \zeta_s }
\right)^{\frac{1}{1+q}},\sigma^* \right\}
\]
\noindent
From the above, we can further derive the inequality below.
\[
\begin{split}
&\E_{a\sim\pi_{\mathcal{H}^q}} Q(s,a)\geq \\
& \frac{\xi_s\zeta_s}{\sqrt{
\xi_s^2+ \min\left\{ \left(\frac{ \alpha (2\pi)^{ \frac{1-q}{2} } (\sigma^{*2}+\xi_s^2)^{\frac{3}{2}} }{\sqrt{q}\xi_s\zeta_s} \right)^{\frac{2}{1+q}},\sigma^{*2} \right\}
}}
\end{split}
\]
\noindent
Consequently,
\[
\begin{split}
& \max_{a\in\mathbb{A}}Q(s,a)-\E_{a\sim\pi_{\mathcal{H}^q}} Q(s,a) \leq \\
& \zeta_s \left(1- \frac{\xi_s}{\sqrt{
\xi_s^2+ \min\left\{ \left(\frac{ \alpha (2\pi)^{ \frac{1-q}{2} } (\sigma^{*2}+\xi_s^2)^{\frac{3}{2}} }{\sqrt{q}\xi_s\zeta_s} \right)^{\frac{2}{1+q}},\sigma^{*2} \right\}
}} \right) \leq \\
& \zeta^* \left(1- \frac{\xi_s}{\sqrt{
\xi_s^2+ \min\left\{ \left(\frac{ \alpha (2\pi)^{ \frac{1-q}{2} } (\sigma^{*2}+\xi_s^2)^{\frac{3}{2}} }{\sqrt{q}\xi_s\zeta_*} \right)^{\frac{2}{1+q}},\sigma^{*2}\right\}
}} \right)
\end{split}
\]
\noindent
Lemma \ref{lemma-1} can be finally obtained as a result.
\end{proof}

\noindent
{\bf Proof of Lemma \ref{lemma-2}}:

\begin{proof}
The proof of Lemma \ref{lemma-2} is similar to the proof of Lemma \ref{lemma-1}.  Specifically, in any $s\in\mathbb{S}$, we have
\[
\begin{split}
& \E_{a\in\pi'(s,\cdot)} Q(s,a)+\alpha\mathcal{H}^{\eta}(\pi'(s,\cdot))=\\
& \frac{\xi_s\zeta_s}{\sqrt{\xi_s^2+\sigma_{\pi,s}^2}}+\frac{\alpha}{1-\eta}\log \left(
\frac{(2\pi)^{\frac{1-\eta}{2}} \sigma_{\pi,s}^{1-\eta} }{ \sqrt{\eta} }
\right)
\end{split}
\]
\noindent
Hence, to determine $\sigma_{\pi_{\mathcal{H}^q}}$ for policy $\pi_{\mathcal{H}^q}$, we can solve the equation below.
\[
\frac{\alpha}{\sigma_{\pi,s}}-\frac{ \xi_s \sigma_{\pi,s} \zeta_s }{ (\xi_s^2+\sigma_{\pi,s}^2)^{\frac{3}{2} } }=0
\]
Although this equation can be solved directly, we choose to instead identify the upper bound of $\sigma_{\pi_{\mathcal{H}^q}}$ by solving the following equation (this is because the solution of the equation below is much easier to analyze than the solution of the equation above)
\[
\frac{\alpha}{\sigma_{\pi,s}}-\frac{ \xi_s \sigma_{\pi,s} \zeta_s }{ (\xi_s^2+\sigma^*)^{\frac{3}{2} } }=0
\]
Similar to our proof of Lemma \ref{lemma-1}, this gives rise to the inequality
\[
\sigma_{\pi_{\mathcal{H}^q}}\leq \min\left\{ \frac{ \sqrt{ \alpha (\sigma^{*2} +\xi_s^2)^{\frac{3}{2}} } }{ \sqrt{\xi_s\zeta_s} },\sigma^* \right\}
\]
Subsequently,
\[
\E_{a\sim\pi_{\mathcal{H}^{\eta}}} Q(s,a)\geq \frac{ \xi_s\zeta_s }{ \sqrt{ \xi_s^2+\min\left\{\frac{ \alpha (\sigma^{*2}+\xi_s^2)^{\frac{3}{2}} }{ \xi_s\zeta_s } ,\sigma^{*2} \right\}} }
\]
\noindent
Therefore,
\[
\begin{split}
& \max_{a\in\mathbb{A}}Q(s,a)-\E_{a\sim\pi_{\mathcal{H}^{\eta}}} Q(s,a) \leq \\
& \zeta_s \left(1- \frac{ \xi_s }{ \sqrt{ \xi_s^2+ \min\left\{ \frac{ \alpha (\sigma^{*2}+\xi_s^2)^{\frac{3}{2}} }{ \xi_s\zeta_s },\sigma^{*2} \right\}  } } \right) \leq \\
&\zeta^* \left(1- \frac{ \xi_s}{ \sqrt{ \xi_s^2+ \min\left\{  \frac{ \alpha (\sigma^{*2}+\xi_s^2)^{\frac{3}{2}} }{ \xi_s\zeta_* },\sigma^{*2} \right\}  } } \right)
\end{split}
\]
Lemma \ref{lemma-2} can now be obtained directly.
\end{proof}

Based on Lemma \ref{lemma-1} and Lemma \ref{lemma-2}, we can proceed to bound the performance difference between standard RL and maximum entropy RL. This is formally presented in Proposition \ref{proposition-app}.
\begin{proposition}
Under assumptions A1, A2 and A3, in comparison to the performance of RL driven by the standard Bellman operator, the performance of RL driven by the maximum entropy Bellman operator is bounded from below by either
\[
\lim_{k\rightarrow\infty}\mathcal{T}_{\mathcal{H}^q}^k Q(s,a)\geq \lim_{k\rightarrow\infty} \mathcal{T}^k Q(s,a)-\frac{\gamma}{1-\gamma} \zeta_{\mathcal{H}_q}
\]
\noindent
or
\[
\lim_{k\rightarrow\infty}\mathcal{T}_{\mathcal{H}^{\eta}}^k Q(s,a)\geq \lim_{k\rightarrow\infty} \mathcal{T}^k Q(s,a)-\frac{\gamma}{1-\gamma} \zeta_{\mathcal{H}_{\eta}}
\]
\noindent
for any state $s\in\mathbb{S}$ and any action $a\in\mathbb{A}$, subject to the entropy measure to be maximized during RL.
\label{proposition-app}
\end{proposition}
\begin{proof}
To prove Proposition \ref{proposition-app}, we must show that
\begin{equation}
\mathcal{T}^k Q(s,a)-\mathcal{T}_{\mathcal{H}}^k Q(s,a)\leq \sum_{j=1}^k \gamma^j \zeta
\label{eq-t-q-diff}
\end{equation}
\noindent
Here $\mathcal{H}$ can be either $\mathcal{H}^q$ or $\mathcal{H}^{\eta}$and $\zeta$ can be either $\zeta_{\mathcal{H}_q}$ or $\zeta_{\mathcal{H}_{\eta}}$. Consider first of all the case when $k=1$. In this case, based on the definition of $\mathcal{T}$ and $\mathcal{T}_{\mathcal{H}}$, we can see that
\[
\begin{split}
& \mathcal{T} Q(s,a)-\mathcal{T}_{\mathcal{H}} Q(s,a) \\
& = \gamma\int_{s'\in\mathbb{S}} P(s,s',a) \left( \max_{a'\in\mathbb{A}} Q(s',a')- \E_{a\sim\pi_{\mathcal{H}}} Q(s,a) \right) \\
& \leq \gamma \zeta
\end{split}
\]
Now assume that \eqref{eq-t-q-diff} holds for $k=1,\ldots,l$. Consider further the situation when $ k=l+1$, i.e.
\[
\begin{split}
& \mathcal{T}^{l+1} Q(s,a)-\mathcal{T}_{\mathcal{H}}^{l+1} Q(s,a) \\
& = \mathcal{T}\mathcal{T}^l Q(s,a) - \mathcal{T}_{\mathcal{H}}^{l+1} Q(s,a) \\
& \leq \mathcal{T} \left( \mathcal{T}_{\mathcal{H}}^l Q(s,a)+\sum_{j=1}^l \gamma^j\zeta \right) - \mathcal{T}_{\mathcal{H}}^{l+1} Q(s,a) \\
& \leq \sum_{j=1}^l \gamma^{j+1}\zeta + \gamma\zeta \\
& = \sum_{j=1}^{l+1} \gamma^j \zeta
\end{split}
\]
\noindent
By mathematical induction, \eqref{eq-t-q-diff} has been verified successfully. Proposition \ref{proposition-app} can now be obtained directly by taking the limit of \eqref{eq-t-q-diff} as $k$ approaches to $\infty$.
\end{proof}

Based on the performance lower bounds developed in Proposition \ref{proposition-app}, the effectiveness of TAC and RAC can be studied further. Specifically, we can consider $\lim_{k\rightarrow\infty}\mathcal{T}^k Q$ as representing the best possible performance achievable through RL, since $\lim_{k\rightarrow\infty} \mathcal{T}^k Q$ converges to $Q^*$ (i.e. the Q-function of the optimal policy $\pi^*$) on RL problems with discrete state space and discrete action space. Consequently, the performance of TAC and RAC is expected to be improved as a result of maximizing the corresponding lower bounds. For example, as $\alpha\rightarrow 0$, the maximum entropy objective in \eqref{eq-lt-cum-rew-ext} vanishes and consequently the lower bounds in Proposition \ref{proposition-app} become $\lim_{k\rightarrow\infty}\mathcal{T}^k Q$. In other words, maximum entropy RL degenerates to conventional RL as we expected.
\begin{figure}[ht!]
\centering
\includegraphics[width=0.8\columnwidth]{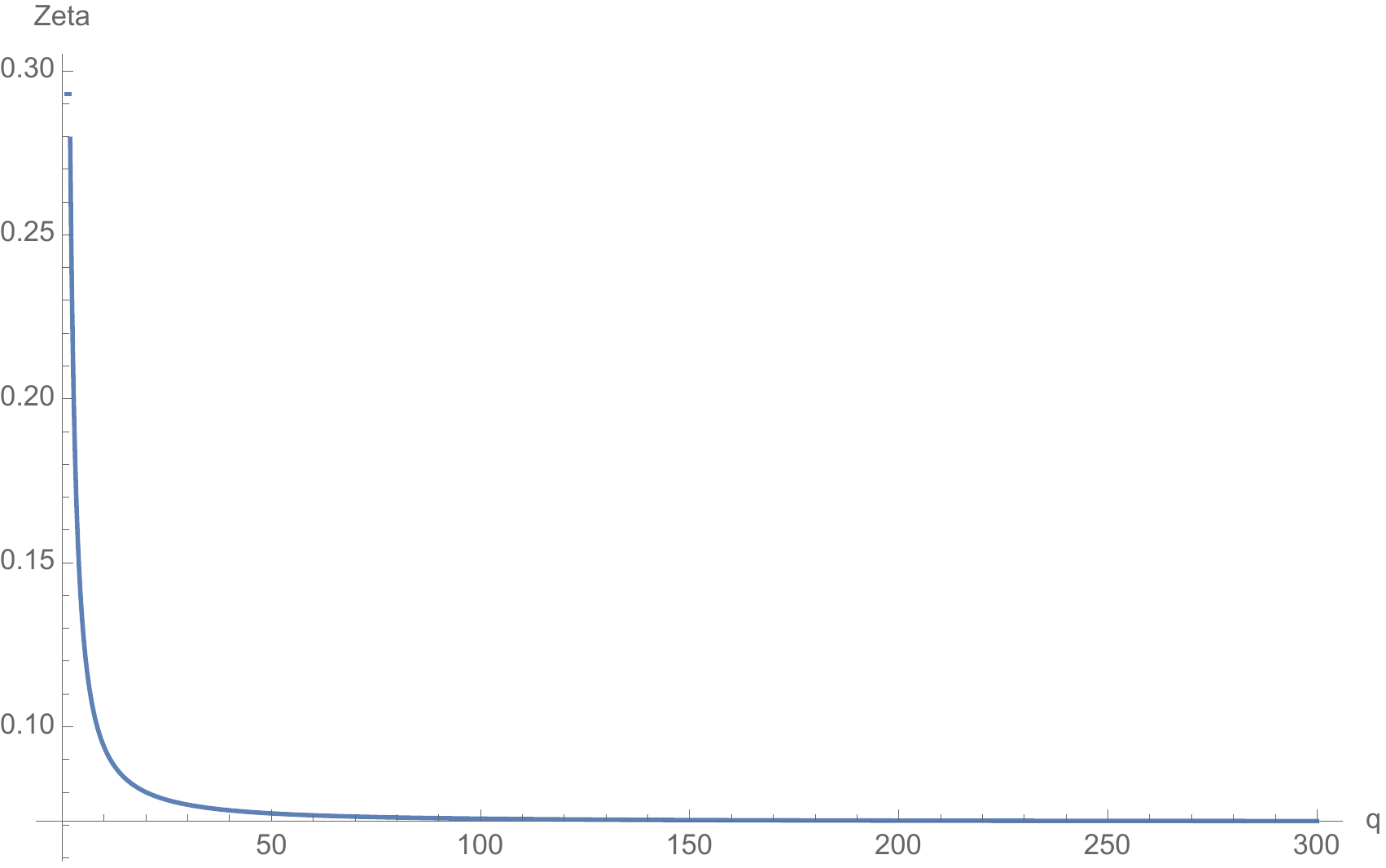}
\caption{The change of $\zeta_{\mathcal{H}_q}$ with respect to $q$.}
\label{fig-zeta}
\end{figure}

In addition, the performance lower bounds can also be influenced by the entropic index. Particularly, for TAC, $\zeta_{\mathcal{H}_q}$ in Lemma \ref{lemma-1} and in Proposition \ref{proposition-app} is a function of the entropic index $q$ of Tsallis entropy. Numerical studies show that $\zeta_{\mathcal{H}_q}$ is often decreasing in value as a result of increasing $q$. An example is depicted in Figure \ref{fig-zeta} where $\xi^*=\xi_*=1.0$, $\zeta^*=\zeta_*=1.0$, $\alpha=1.0$ and $\sigma^*=1.0$. Since TAC degenerates to SAC when $q=1$, we can conclude that TAC can achieve a higher performance bound than SAC through proper setting of $q>1$. This result suggests that TAC can be potentially more effective than SAC.

Different from TAC, the performance lower bound for RAC in Proposition \ref{proposition-app} is not a direct function of the entropic index $\eta$ of R\'enyi entropy. Nevertheless, with different settings of $\eta$, we can adjust the influence of the maximum entropy objective in \eqref{eq-lt-cum-rew-ext} and in \eqref{eq-pi-improve-gac}. Through this way, policy improvement during RL can be controlled further, resulting in varied trade-offs between exploration and exploitation. It paves the way for training simultaneously multiple policies through our Ensemble Actor-Critic (EAC) algorithm developed in Subsection \ref{sub-sec-eac}.

\section*{Appendix G}

Proof of Proposition \ref{proposition-ens} is presented in this appendix.
\begin{proof}
For the sake of simplicity, we will only consider the case when assumption A1 in Appendix F is valid. The proof can be extended to RL problems with multi-dimensional actions. However, the details will be omitted in this appendix.

Because $Q^\pi$ is continuous, differentiable with bounded derivatives and unimodal in any state $s\in\mathbb{S}$, by adjusting $\alpha$ in \eqref{eq-pi-improve-gac}, we can bound the difference between $\bar{a}_{\pi_i,s}$ and the action that maximizes $Q^{\pi}(s,\cdot)$, where $\bar{a}_{\pi_i,s}$ refers to the mean action of the stochastic policy $\pi_i(s,\cdot)$ and $\pi_i$ (with $i=1,\ldots,L$) in Proposition \ref{proposition-ens} denotes the $i$-th policy in the ensemble that maximizes \eqref{eq-pi-improve-gac} when $\alpha\neq 0$. In view of this, we will specifically consider the case when $\bar{a}_{\pi_i,s}$ equals to the Q-function maximization action. The case when they are different but the difference is bounded can be analyzed similarly.

In line with the above, we can proceed to study the behavior of the joint policy $\pi^e$ in any state $s$. For any policy $\pi\in\Pi$, the standard deviation for action sampling in state $s$ is $\sigma_{\pi,s}$ and $\sigma_{\pi,s}\leq\sigma^*$. Because of this, when $L$ is large enough, we can see that
\[
\E_{a\sim\pi^e(s,\cdot)} Q^{\pi}(s, a)\geq \E_{a\sim\tilde{\pi}^e(s,\cdot)}Q^{\pi} (s,a)
\]
\noindent
where
\[
\tilde{\pi}^e(s) = \argmax_{a\in\{\hat{a}_1,\ldots,\hat{a}_L\}} Q^{\pi}(s,a)
\]
\noindent
is a policy that chooses the action with the highest Q-value out of $L$ actions sampled independently from the same normal distribution with mean
\[
\argmax_{a\in\mathbb{A}} Q^{\pi}(s,a)=0
\]
and standard deviation $\sigma^*$. Here, without loss of generality, we simply consider the situation that $Q^{\pi}(s,\cdot)$ is maximized when $a=0$. Let $|\tilde{\pi}^e(s)|$ stand for the absolute value of actions sampled from policy $\tilde{\pi}^e$. Since $Q^{\pi}$ is continuous, differentiable with bounded derivatives and unimodal, when $L$ is sufficiently large, $-|\tilde{\pi}^e(s)|$ can be estimated accurately as the maximum over $L$ actions sampled independently from the same distribution $\Omega$ with the probability density function below.
\[
\Omega(a)=\left\{\begin{array}{cc}
\frac{2 \exp(\frac{-a^2}{2 \sigma^{*2}})}{ \sqrt{2\pi} \sigma^* }, & a\leq 0 \\
0, & a>0
\end{array}
\right.
\]
\noindent
According to the Fisher–Tippett–Gnedenko theorem \cite{fisher1930book}, with large $L$, the distribution for $-|\tilde{\pi}^e(s)|$ can be approximated precisely by the distribution with the following \emph{cumulative distribution function} (CDF).
\begin{equation}
F_{\Omega^{(L)}}(a)\approx F_{EV}\left( \frac{a-\rho_L}{\varphi_L} \right)
\label{eq-dist-omega-approx}
\end{equation}
\noindent
where $F_{\Omega^{(L)}}$ refers to the CDF of the probability distribution determined by the maximum over $L$ randomly sampled numbers, denoted as $\Omega^{(L)}$. $F_{EV}$ stands for the CDF of the extreme value distribution \cite{leadbetter2012book}. Meanwhile
\[
\rho_L=F^{(-1)}_{\Omega}(1-\frac{1}{L})
\]
\noindent
\[
\varphi_L=F^{(-1)}_{\Omega}(1-\frac{1}{eL})-F^{(-1)}_{\Omega}(1-\frac{1}{L})
\]
\noindent
Here $F^{(-1)}_{\Omega}$ denotes the inverse of the CDF of distribution $\Omega$. Based on \eqref{eq-dist-omega-approx}, it is possible to analytically estimate the mean and standard deviation of distribution $\Omega^{(L)}$. However the resulting mathematical expressions are fairly complicated and will not be presented in this appendix. Instead, all we need is to examine the limits below.
\[
\lim_{L\rightarrow\infty}\E\left[ \Omega^{(L)} \right]=0
\]
\[
\lim_{L\rightarrow\infty} \var\left[ \Omega^{(L)} \right]=0
\]
\noindent
Due to the two limits above, when $L\rightarrow\infty$, we can draw the conclusion that
\[
\begin{split}
& Q^{\pi^e}(s,a)-Q^{\pi'}(s,a)\\
& =\gamma \E_{s'\sim P(s,s',a)}\left( \begin{array}{l} \lim_{L\rightarrow\infty} \E_{a\sim\pi^e(s',\cdot)} Q^{\pi}(s', a)- \\ \max_{\pi'\in\Pi} \E_{a\sim\pi'(s',\cdot)} Q^{\pi}(s',a)\end{array} \right)\\
& \geq\gamma \E_{s'\sim P(s,s',a)} \left(\begin{array}{l} \lim_{L\rightarrow\infty} \E_{a\sim\tilde{\pi}^e(s,\cdot)} Q^{\pi}(s', a) -\\ \max_{\pi'\in\Pi} \E_{a\sim\pi'(s',\cdot)} Q^{\pi}(s',a)\end{array} \right) \\
& =\gamma\E_{s'\sim P(s,s',a)} \left(\begin{array}{l}\lim_{L\rightarrow\infty} \E_{a\sim\Omega^{(L)}} Q^{\pi}(s', a)-\\ \max_{\pi'\in\Pi} \E_{a\sim\pi'(s',\cdot)} Q^{\pi}(s',a) \end{array}\right)\\
& =\gamma\E_{s'\sim P(s,s',a)} \left(\begin{array}{l} \max_{a\in(-\infty,\infty)} Q^{\pi}(s', a)\\-\max_{\pi'\in\Pi} \E_{a\sim\pi'(s',\cdot)} Q^{\pi}(s',a)\end{array} \right)\\
& \geq 0
\end{split}
\]
This proves Proposition \ref{proposition-ens}.
\end{proof}

\begin{table*}[!ht]
\caption{Hyper-parameters settings of SAC, TAC, RAC, EAC-TAC and EAC-RAC.}
\label{tab-hyp-para}
\centering
\scalebox{0.80}{
\begin{tabular}{c|l|c|c|c|c|c|c}
\multirow{2}{*}{\textbf{Algorithms}} & \multirow{2}{*}{\textbf{Hyper-Parameters}} & \multicolumn{6}{c }{\textbf{Problems}}                                                                                                                                                                                                                                               \\ \cline{3-8}
                                     &                                            & \textit{\textbf{Ant}}              & \multicolumn{1}{c|}{\textit{\textbf{Half Cheetah}}} & \multicolumn{1}{c|}{\textit{\textbf{Hopper}}} & \multicolumn{1}{c|}{\textit{\textbf{Lunar Lander}}} & \multicolumn{1}{c|}{\textit{\textbf{Reacher}}} & \textit{\textbf{Walker2D}}         \\ \hline
\multirow{3}{*}{\textit{SAC}}        & Reward Scale                               & 3.0                                & 3.0                                                 & 3.0                                           & 3.0                                                 & 3.0                                            & 3.0                                \\
                                     & $\alpha$                              & 1.0                                & 1.0                                                 & 1.0                                           & 1.0                                                 & 1.0                                            & 1.0                                \\
                                     & Gradient Steps                             & 1                                  & 1                                                   & 4                                             & 1                                                   & 4                                              & 1                                  \\ \hline
\multirow{4}{*}{\textit{TAC}}        & Reward Scale                               & 3.0                                & 3.0                                                 & 3.0                                           & 3.0                                                 & 3.0                                            & 3.0                                \\
                                     & $\alpha$                              & 0.8                                & 0.8                                                 & 0.8                                           & 0.8                                                 & 0.8                                            & 0.8                                \\
                                     & Gradient Steps                             & 1                                  & 1                                                   & 4                                             & 1                                                   & 4                                              & 1                                  \\
                                     & Entropic index $q$                                  & $\{1.5,2.0,2.5\}$                                & $\{1.5,2.0,2.5\}$                                                 & $\{1.5,2.0,2.5\}$                                           & $\{1.5,2.0,2.5\}$                                                 & $\{1.5,2.0,2.5\}$                                            & $\{1.5,2.0,2.5\}$                                \\ \hline
\multirow{4}{*}{\textit{RAC}}        & Reward Scale                               & 3.0                                & 3.0                                                 & 3.0                                           & 3.0                                                 & 3.0                                            & 3.0                                \\
                                     & $\alpha$                              & 0.8                                & 0.8                                                 & 0.8                                           & 0.8                                                 & 0.8                                            & 0.8                                \\
                                     & Gradient Steps                             & 1                                  & 1                                                   & 4                                             & 1                                                   & 4                                              & 1                                  \\
                                     & Entropic index $\eta$                                    & $\{1.5,2.0,2.5\}$                                & $\{1.5,2.0,2.5\}$                                                 & $\{1.5,2.0,2.5\}$                                           & $\{1.5,2.0,2.5\}$                                                 & $\{1.5,2.0,2.5\}$                                            & $\{1.5,2.0,2.5\}$                                \\ \hline
\multirow{4}{*}{\textit{EAC-TAC}}    & Reward Scale                               & 3.0                                & 3.0                                                 & 3.0                                           & 3.0                                                 & 3.0                                            & 3.0                                \\
                                     & $\alpha$                              & 0.6                                & 0.6                                                 & 0.6                                           & 0.6                                                 & 0.6                                            & 0.6                                \\
                                     & Gradient Steps                             & 1                                  & 1                                                   & 4                                             & 1                                                   & 4                                              & 1                                  \\
                                     & Entropic index $q$           & 2.0 & 1.5                  & 1.5            & 1.5                  & 1.5             & 1.5 \\
                                     & Ensemble size & 6 & 6 & 6 & 6 & 6 & 6 \\ \hline
\multirow{4}{*}{\textit{EAC-RAC}}    & Reward Scale                               & 3.0                                & 3.0                                                 & 3.0                                           & 3.0                                                 & 3.0                                            & 3.0                                \\
                                     & $\alpha$                              & 0.6                                & 0.6                                                 & 0.6                                           & 0.6                                                 & 0.6                                            & 0.6                                \\
                                     & Gradient Steps                             & 1                                  & 1                                                   & 4                                             & 1                                                   & 4                                              & 1                                  \\
                                     & Entropic index $\eta$               & 1.5 & 1.5                  & 1.5            & 2.0                  & 1.5             & 1.5 \\
                                     & Ensemble size & 6 & 6 & 6 & 6 & 6 & 6 \\ \hline
\end{tabular}}
\end{table*}
  \begin{figure*}[!ht]
      \begin{minipage}[t]{0.33\textwidth}
        \includegraphics[width=\textwidth]{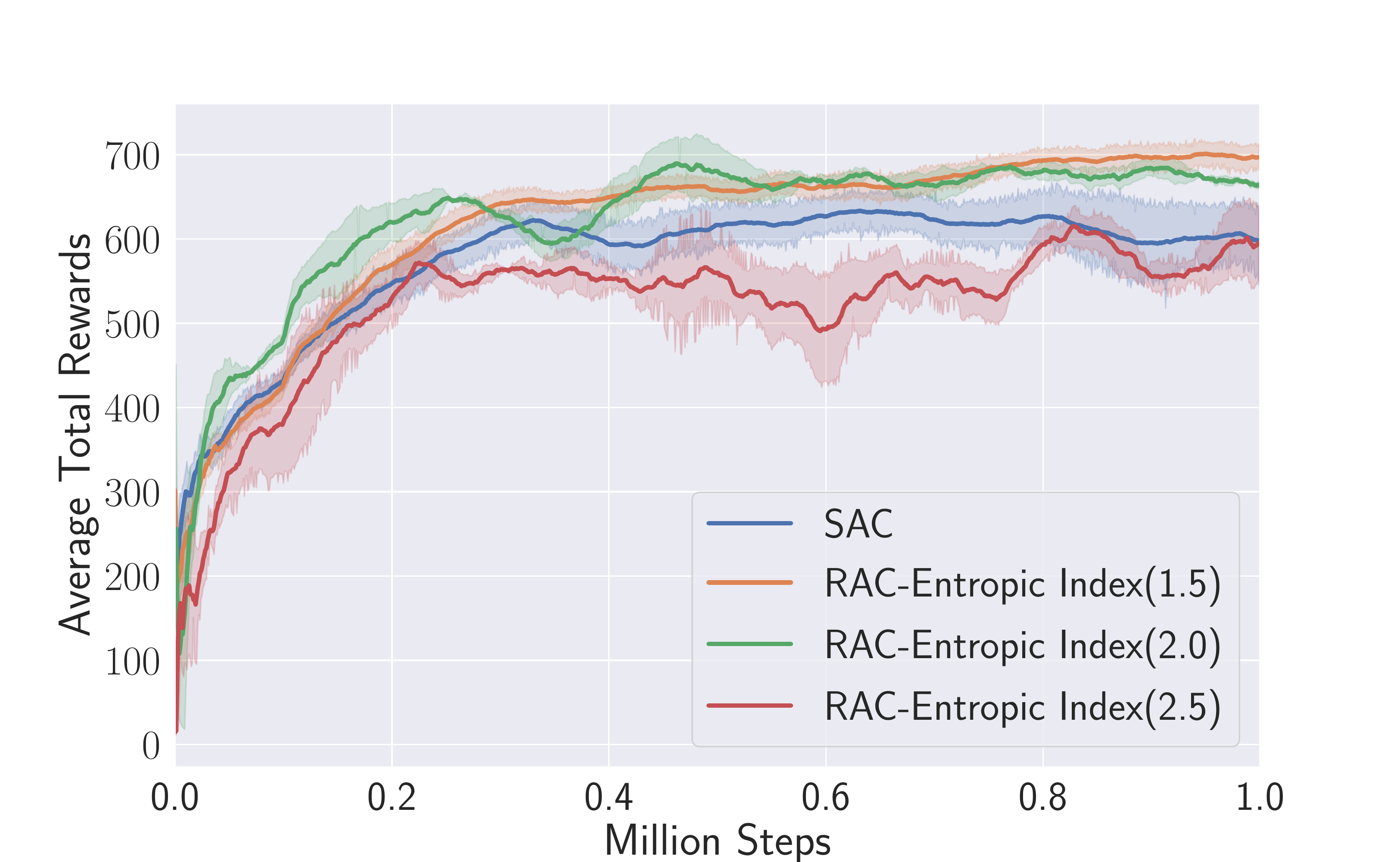}
        \subcaption{\tiny{Ant}}
      \end{minipage}%
      \begin{minipage}[t]{0.33\textwidth}
        \includegraphics[width=\textwidth]{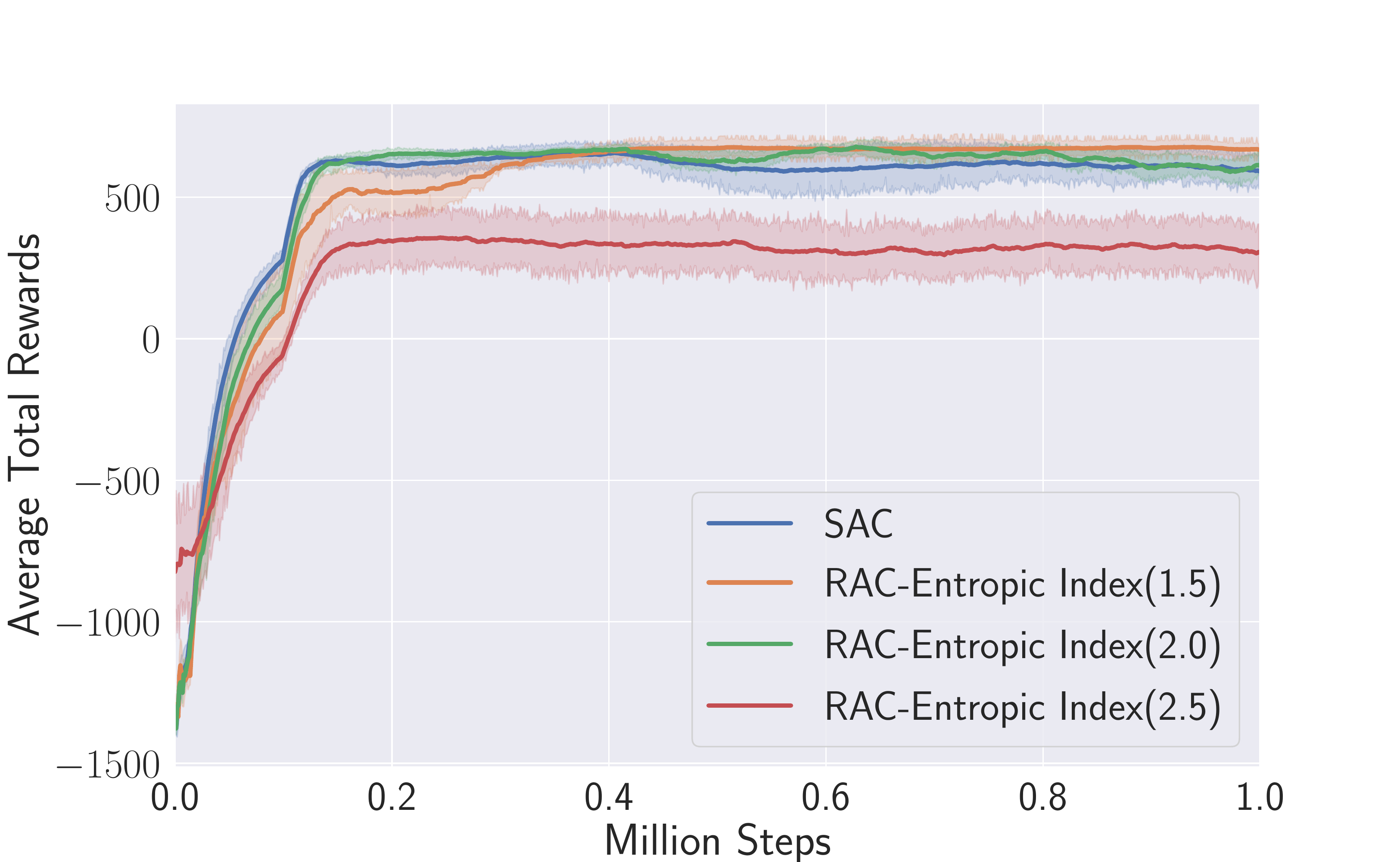}
        \subcaption{\tiny{Half Cheetah}}
      \end{minipage}
      \begin{minipage}[t]{0.33\textwidth}
        \includegraphics[width=\textwidth]{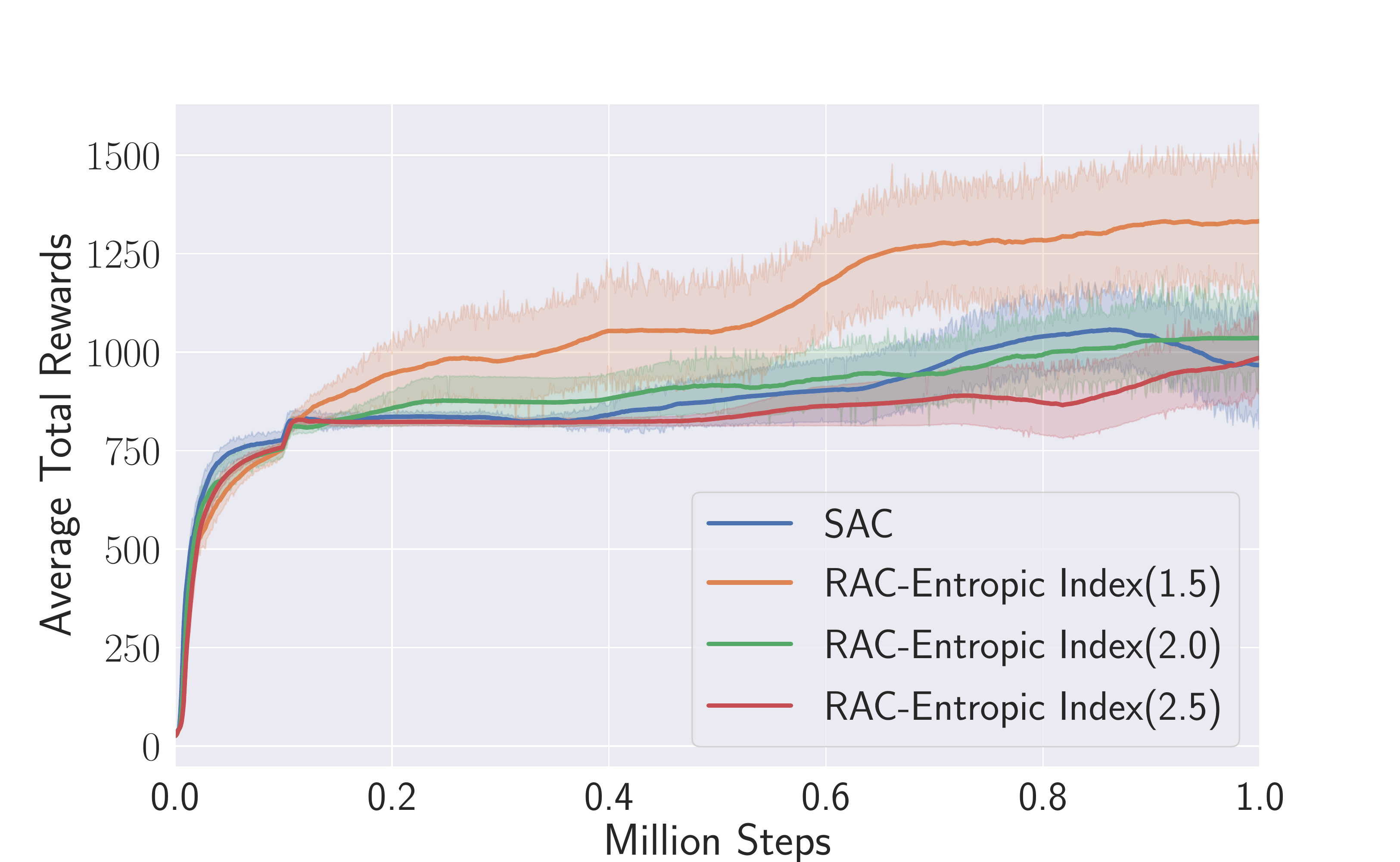}
        \subcaption{\tiny{Hopper}}
      \end{minipage}
      \\
      \begin{minipage}[t]{0.33\textwidth}
        \includegraphics[width=\textwidth]{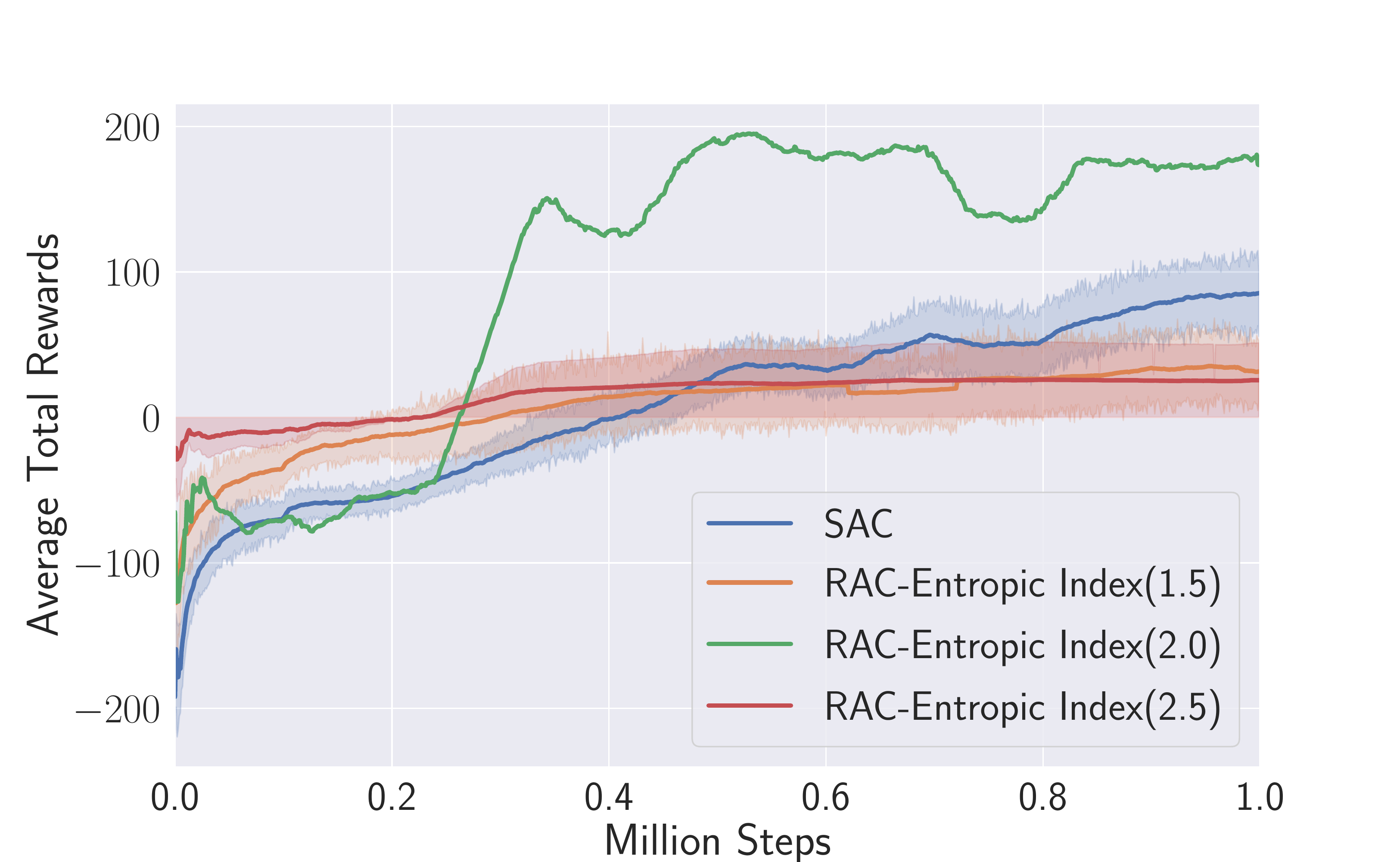}
        \subcaption{\tiny{Lunar Lander}}
      \end{minipage}
      \begin{minipage}[t]{0.33\textwidth}
        \includegraphics[width=\textwidth]{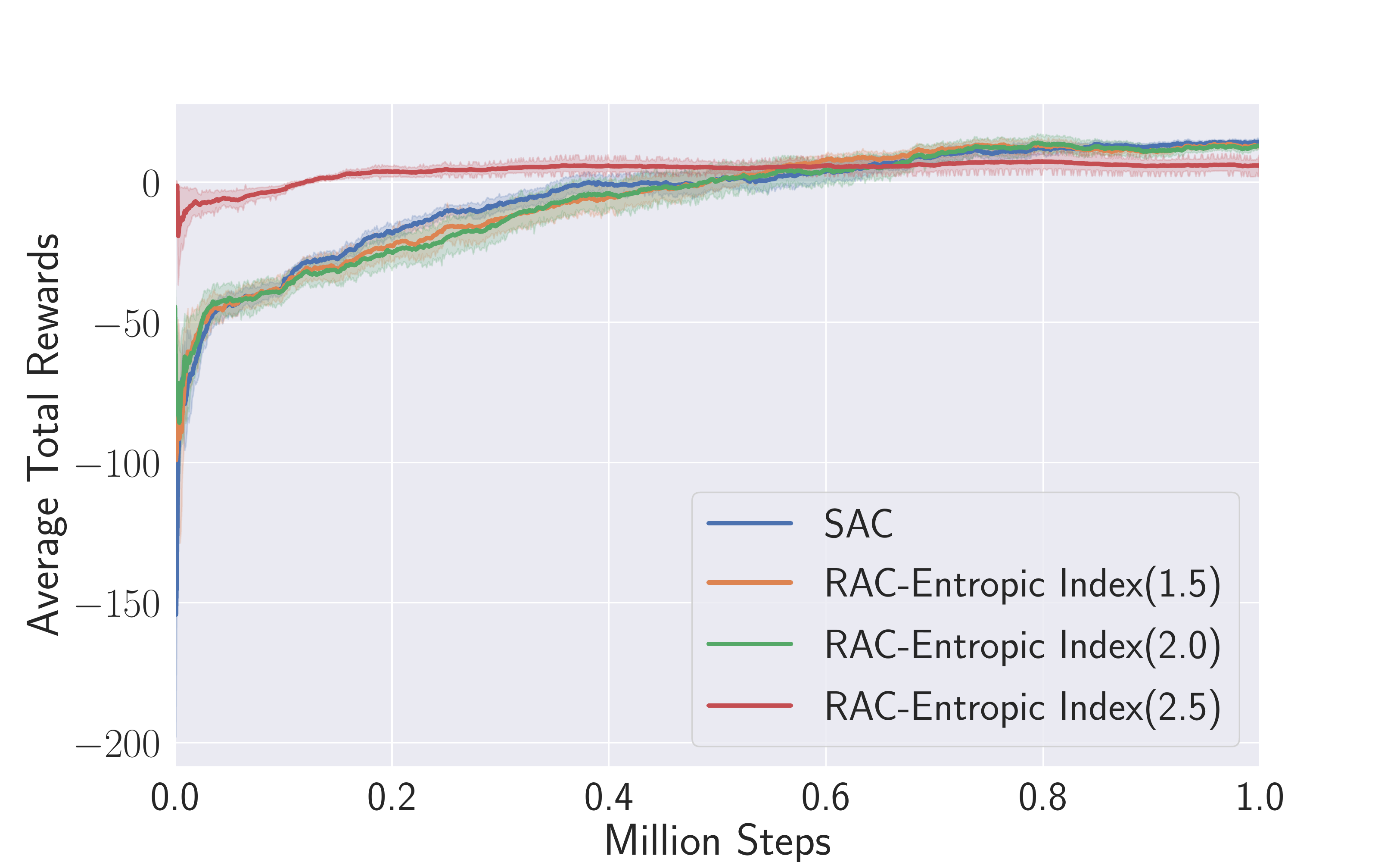}
        \subcaption{\tiny{Reacher}}
      \end{minipage}
      \begin{minipage}[t]{0.33\textwidth}
        \includegraphics[width=\textwidth]{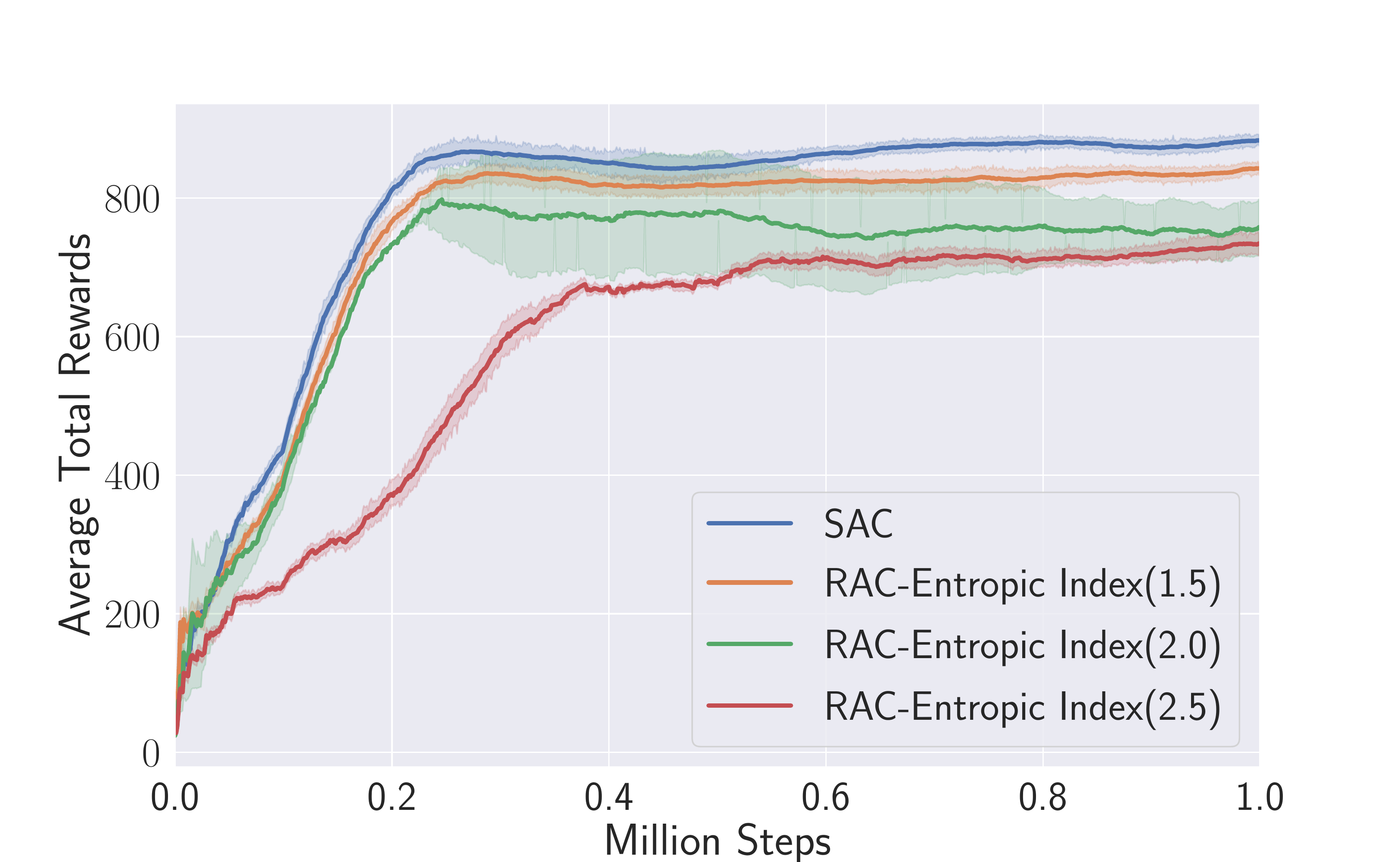}
        \subcaption{\tiny{Walker2D}}
      \end{minipage}
      \caption{The influences of entropic indices ($[1.5, 2.0, 2.5]$) on the performance of RAC on six benchmark control problems, with SAC serving as the baseline.}
      \label{fig:rac_entropic_index}
    \end{figure*}
\section*{Appendix H}

This appendix details the hyper-parameter settings for all competing algorithms involved in our experimental study in this paper. For SAC, TAC, RAC, EAC-TAC and EAC-RAC, the value functions $V_{\theta}$ and $Q_{\omega}$ as well as the policy $\pi_{\phi}$ are implemented as DNNs with two hidden layers. Each hidden layer contains 128 ReLU hidden units. The same network architectures have also been adopted in \cite{haarnoja2018icml}.

For each benchmark control task, we follow closely the hyper-parameter settings of SAC reported in \cite{haarnoja2018icml}. Specifically, SAC introduces a new hyper-parameter named the \emph{reward scale} which controls the scale of the step-wise reward provided by a learning environment as feedback to an RL agent. According to \cite{haarnoja2018icml}, the performance of SAC is highly sensitive to the reward scale since it affects the stochasticity of the learned policies. In addition, we should determine carefully the number of gradient-based updating steps, i.e. \emph{gradient steps}, to be performed during each learning iteration in Algorithm \ref{algorithm-1}. This hyper-parameter may also affect the learning performance to a certain extent. Relevant hyper-parameter settings with respect to SAC, TAC, RAC, EAC-TAC and EAC-RAC have been summarized in Table \ref{tab-hyp-para}.

In addition to the hyper-parameters covered in Table \ref{tab-hyp-para}, the learning rate for all algorithms and on all benchmark problems have been set consistently to $3\times 10^{-4}$. The discount factor $\gamma=0.99$. The smooth co-efficient for updating the target value function $\tau=0.01$. Additional hyper-parameter settings for TRPO, PPO and ACKTR are outlined here too. Specifically both the critic and actor of TRPO, PPO and ACKTR are implemented as DNNs with $64 \times 64$ hidden units, following the recommendation in~\cite{schulman20171}. Meanwhile the number of gradient-based training steps for each learning iteration equals to $10$ in the three competing algorithms. In addition, TRPO and PPO have adopted the Generalized Advantage Estimation (GAE)~\cite{Schulman:2015vz} technique with hyper-parameter $\lambda=0.95$. The clipping factor $\varepsilon$ of PPO is set to $0.2$.

To verify the true difference in performance and sample efficiency, we have evaluated all algorithms for 1 million consecutive steps on each benchmark problem. The testing performance of trained policies has also been recorded after obtaining every 1,000 state-transition samples from the learning environment. Meanwhile, we have chosen 10 random seeds for running each algorithm on every benchmark to reveal the performance difference among all competing algorithms with statistical significance.

\section*{Appendix I}

In this appendix, the influence of entropic index $\eta$ on the performance of RAC is studied empirically. In particular, Figure~\ref{fig:rac_entropic_index} depicts the observed learning performance of RAC subject to three different entropic index settings (i.e, 1.5, 2.0, and 2.5) and also compares RAC with SAC as the baseline. From the figure, we can clearly see that, with proper settings of the entropic index $\eta$, RAC can perform consistently well across all benchmarks and sometimes outperform SAC. Meanwhile, it seems that smaller values for $\eta$ usually yield better performance. However this is not always true. For example, on the Lunar Lander problem, the best performance of RAC is obtained when $\eta=2.0$. In consequence, we can draw two conclusions from the results. First, $\eta$ can be set to small values for general good performance. Second, the influence of $\eta$ on the performance of RAC can be problem-specific. Due to this reason, while comparing RAC with other competing algorithms in Figure \ref{fig:performance_evaluation}, we used the best $\eta$ value observed in Figure \ref{fig:rac_entropic_index} for each benchmark.

\end{document}